\documentclass{article}

\usepackage{microtype}
\usepackage{graphicx}
\usepackage{subcaption}
\usepackage{booktabs} 

\usepackage{balance}

\usepackage{bm,xspace}

\usepackage[table]{xcolor}

\newcommand{\wcard}{\text{*}}

\newcommand{\update}[1]{{\color{black}#1}}

\newcommand{\erdosrenyi}{Erd{\H o}s-R\'enyi\xspace}

\newcommand*\dbar[1]{\overline{\overline{\lower0.2ex\hbox{$#1$}}}}

\newcommand{\transpose}{^{\text{T}}}

\newcommand{\graphletCounting}{$\Gamma_\text{1-hot}$}
\newcommand{\karyGnn}{$\Gamma_\text{GIN}$}
\newcommand{\karyRpGnn}{$\Gamma_\text{RPGIN}$}

\usepackage{tikz}
\usetikzlibrary{shapes.misc}
\colorlet{darkgreen}{green!50!black}

\newcommand*{\Kfourbrgg}{{\scriptsize
\tikz[scale=0.1]{
\draw[thin] (2,0) -- (0,0) -- (0,2) -- (2,2) -- (0,0);
\draw[fill=darkgreen,draw=darkgreen] (0,0) circle (12pt);
\draw[fill=red,draw=red] (0,2) circle (12pt);
\draw[fill=darkgreen,draw=darkgreen] (2,2) circle (12pt);
\draw[fill=blue,draw=blue] (2,0) circle (12pt);
}
}}
\newcommand*{\Kfourbrgr}{{\scriptsize
\tikz[scale=0.1]{
\draw[thin] (2,0) -- (0,0) -- (0,2) -- (2,2) -- (0,0);
\draw[fill=darkgreen,draw=darkgreen] (0,0) circle (12pt);
\draw[fill=red,draw=red] (0,2) circle (12pt);
\draw[fill=red,draw=red] (2,2) circle (12pt);
\draw[fill=blue,draw=blue] (2,0) circle (12pt);
}
}}

\newcommand*{\VeeNbrg}{{\scriptsize
\tikz[scale=0.1]{
\draw[thin] (0,2) -- (2,2) -- (1,0);
\draw[fill=darkgreen,draw=darkgreen] (1,0) circle (12pt);
\draw[fill=red,draw=red] (0,2) circle (12pt);
\draw[fill=blue,draw=blue] (2,2) circle (12pt);
}
}}

\newcommand*{\VeeNbrr}{{\scriptsize
\tikz[scale=0.1]{
\draw[thin] (0,2) -- (2,2) -- (1,0);
\draw[fill=red,draw=red] (1,0) circle (12pt);
\draw[fill=red,draw=red] (0,2) circle (12pt);
\draw[fill=blue,draw=blue] (2,2) circle (12pt);
}
}}

\newcommand*{\TrigNbrg}{{\scriptsize
\tikz[scale=0.1]{
\draw[thin] (0,2) -- (2,2) -- (1,0) -- (0,2);
\draw[fill=darkgreen,draw=darkgreen] (1,0) circle (12pt);
\draw[fill=red,draw=red] (0,2) circle (12pt);
\draw[fill=blue,draw=blue] (2,2) circle (12pt);
}
}}

\newcommand{\allpossibleGraphsK}{\mathcal{F}_{\leq k}}
\newcommand{\subgraph}{F_{k'}}
\newcommand{\samestructure}{\cH(F_{k'})}

\newcommand{\harrow}[1]{\mathstrut\mkern2.5mu#1\mkern-11mu\raise1.6ex%
  \hbox{$\scriptscriptstyle\rightharpoonup$}}

\usepackage{amsmath,amsfonts,bm,dsfont,diagbox,amsthm}

\makeatletter
\let\save@mathaccent\mathaccent
\newcommand*\if@single[3]{%
    \setbox0\hbox{${\mathaccent"0362{#1}}^H$}%
    \setbox2\hbox{${\mathaccent"0362{\kern0pt#1}}^H$}%
    \ifdim\ht0=\ht2 #3\else #2\fi
    }
\newcommand*\rel@kern[1]{\kern#1\dimexpr\macc@kerna}
\newcommand*\widebar[1]{{\@ifnextchar^{{\wide@bar{#1}{0}}}{\wide@bar{#1}{1}}}}
\newcommand*\wide@bar[2]{\if@single{#1}{\wide@bar@{#1}{#2}{1}}{\wide@bar@{#1}{#2}{2}}}
\newcommand*\wide@bar@[3]{%
\begingroup
\def\mathaccent##1##2{%
    \let\mathaccent\save@mathaccent
    \if#32 \let\macc@nucleus\first@char \fi
    \setbox\z@\hbox{$\macc@style{\macc@nucleus}_{}$}%
    \setbox\tw@\hbox{$\macc@style{\macc@nucleus}{}_{}$}%
    \dimen@\wd\tw@
    \advance\dimen@-\wd\z@
    \divide\dimen@ 3
    \@tempdima\wd\tw@
    \advance\@tempdima-\scriptspace
    \divide\@tempdima 10
    \advance\dimen@-\@tempdima
    \ifdim\dimen@>\z@ \dimen@0pt\fi
    \rel@kern{0.6}\kern-\dimen@
    \if#31
        \overline{\rel@kern{-0.6}\kern\dimen@\macc@nucleus\rel@kern{0.4}\kern\dimen@}%
        \advance\dimen@0.4\dimexpr\macc@kerna
        \let\final@kern#2%
        \ifdim\dimen@<\z@ \let\final@kern1\fi
        \if\final@kern1 \kern-\dimen@\fi
    \else
        \overline{\rel@kern{-0.6}\kern\dimen@#1}%
    \fi
}%
\macc@depth\@ne
\let\math@bgroup\@empty \let\math@egroup\macc@set@skewchar
\mathsurround\z@ \frozen@everymath{\mathgroup\macc@group\relax}%
\macc@set@skewchar\relax
\let\mathaccentV\macc@nested@a
\if#31
    \macc@nested@a\relax111{#1}%
\else
    \def\gobble@till@marker##1\endmarker{}%
    \futurelet\first@char\gobble@till@marker#1\endmarker
    \ifcat\noexpand\first@char A\else
        \def\first@char{}%
    \fi
    \macc@nested@a\relax111{\first@char}%
\fi
\endgroup
}
\makeatother

\newtheorem{definition}{Definition}

\def\eqref#1{equation~\ref{#1}}

\def\1{\bm{1}}

\def\vone{{\bm{1}}}

\def\vh{{\bm{h}}}

\def\mP{{\bm{P}}}

\def\mW{{\bm{W}}}
\def\mX{{\bm{X}}}

\DeclareMathAlphabet{\mathsfit}{\encodingdefault}{\sfdefault}{m}{sl}
\SetMathAlphabet{\mathsfit}{bold}{\encodingdefault}{\sfdefault}{bx}{n}

\def\cF{{\mathcal{F}}}
\def\cG{{\mathcal{G}}}
\def\cH{{\mathcal{H}}}

\def\sD{{\mathbb{D}}}

\def\sR{{\mathbb{R}}}

\def\sX{{\mathbb{X}}}
\def\sY{{\mathbb{Y}}}

\newcommand{\E}{\mathbb{E}}

\newcommand{\Wvar}{W}

\usepackage{amsmath}

\usepackage{hyperref}
\usepackage{url}
\usepackage[inline]{enumitem}
\usepackage{bbm}
\usepackage{mathtools}
\usepackage{cleveref}
\usepackage{thmtools} 
\usepackage{thm-restate}
\usepackage{lipsum}
\usepackage{textcomp}

\usepackage[accepted]{icml2021}

\icmltitlerunning{Size-Invariant Graph Representations for Graph Classification Extrapolations}

\begin{document}

\twocolumn[
\icmltitle{Size-Invariant Graph Representations for Graph Classification Extrapolations}



\icmlsetsymbol{equal}{*}

\begin{icmlauthorlist}
\icmlauthor{Beatrice Bevilacqua}{equal,cs}
\icmlauthor{Yangze Zhou}{equal,stat}
\icmlauthor{Bruno Ribeiro}{cs}
\end{icmlauthorlist}

\icmlaffiliation{cs}{Department of Computer Science, and}
\icmlaffiliation{stat}{Department of Statistics, Purdue University, West Lafayette, Indiana, USA}

\icmlcorrespondingauthor{Beatrice Bevilacqua}{bbevilac@purdue.edu}

\icmlkeywords{Graph Representation Learning, Graph Neural Networks, Graph Classification, Causal Extrapolation}

\vskip 0.3in
]



\printAffiliationsAndNotice{\icmlEqualContribution} 

\begin{abstract}
In general, graph representation learning methods assume that the train and test data come from the same distribution. 
In this work we consider an underexplored area of an otherwise rapidly developing field of graph representation learning:
The task of out-of-distribution (OOD) graph classification, where train and test data have different distributions, with test data unavailable during training.
Our work shows it is possible to use a causal model to learn approximately invariant representations that better extrapolate between train and test data.
Finally, we conclude with synthetic and real-world dataset experiments showcasing the benefits of representations that are invariant to train/test distribution shifts.
\end{abstract}

\section{Introduction}
\label{introduction}

In general, graph representation learning methods assume that the train and test data come from the same distribution. 
Unfortunately, this assumption is not always valid in real-world deployments~\citep{hu2020open,koh2020wilds,d2020underspecification}.
When the test distribution is different from training, the test data is described as {\em out of distribution (OOD)}.
Differences in train/test distribution may be due to environmental factors such as those related to the way the data is collected or processed.

Particularly, in graph classification tasks, where $\cG$ is the graph and $Y$ its label, we often see different graph sizes and/or distinct arrangements of vertex attributes associated with the same target label.
{\em How should we learn a graph representation for out-of-distribution inductive tasks (extrapolations), where the graphs in training and test (deployment) have distinct characteristics (i.e., ${\rm P}^\text{tr}(\cG) \neq {\rm P}^\text{te}(\cG)$)? }
Are inductive graph neural networks (GNNs) robust to distribution shifts between ${\rm P}^\text{tr}(\cG)$ and ${\rm P}^\text{te}(\cG)$?
If not, is it possible to design a graph classifier that is robust to such OOD shifts without access to samples from ${\rm P}^\text{te}(\cG)$?

\begin{figure}[t!!]
    \centering
    \begin{minipage}{\columnwidth}
    \centering
    \includegraphics[width=3.2in]{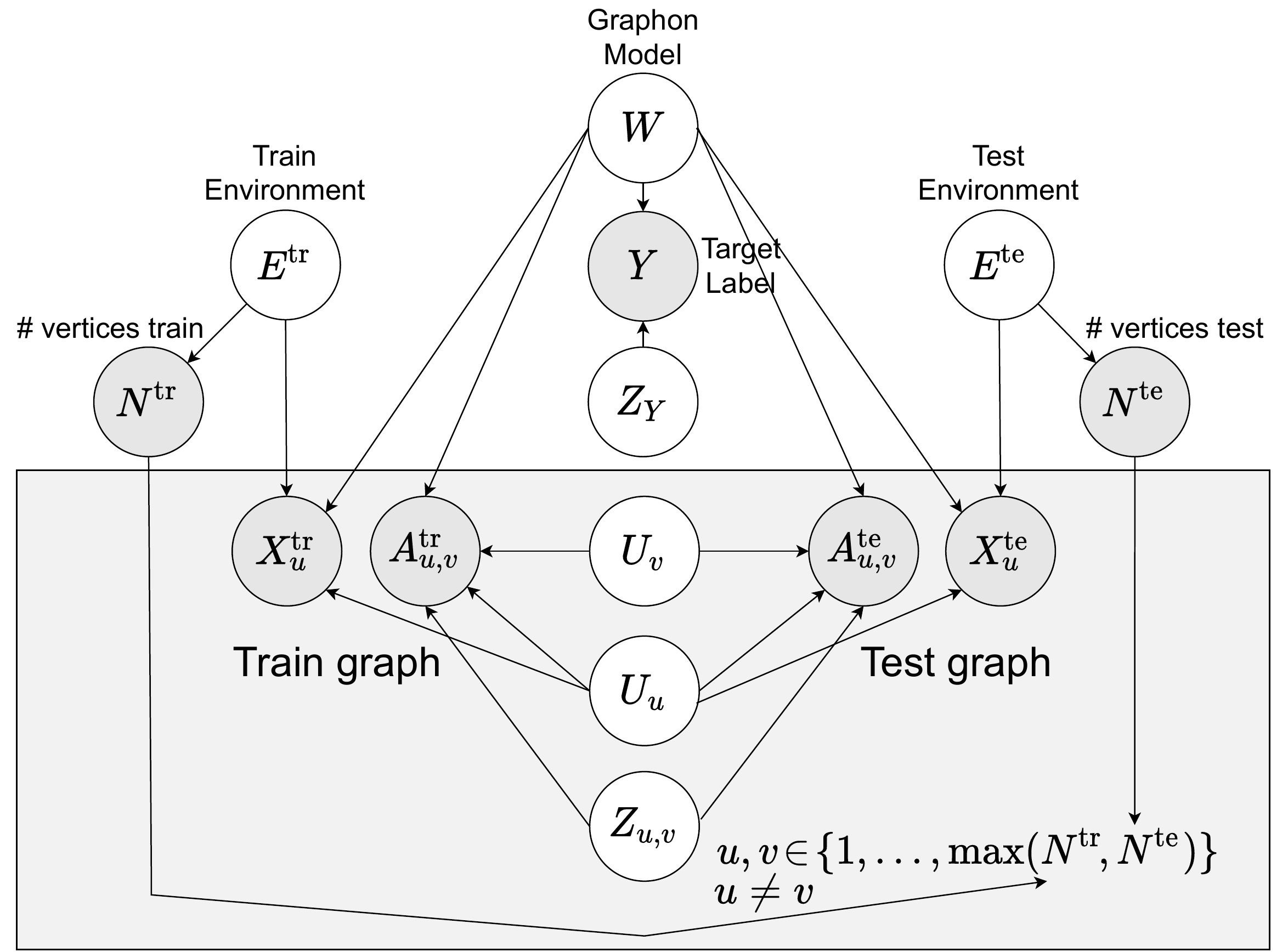}\hfill
    \end{minipage}
    \caption{The twin network DAG~\citep{balke1994twinnets} of our structural causal model (SCM). Gray (resp.\ white) vertices represent observed (resp.\ hidden) random variables.}
    \label{fig:SCM}
\vspace{-10pt}
\end{figure}

In this work we consider an OOD graph classification task with different train and test distributions based on graph sizes and vertex attributes. Our work focuses on simple (no self-loops) undirected graphs with discrete vertex attributes.
We make the common assumption of independence between cause and mechanisms~\citep{bengio2019meta,besserve2018group,johansson2016learning,louizos2017causal,raj2020causal,scholkopf2019causality,arjovsky2019invariant}, which states that ${\rm P}(Y|\cG)$ remains the same between train and test.
We also assume we do not have access to samples from ${\rm P}^\text{te}(\cG)$, hence covariate shift adaptation methods (such as~\citet{yehudai2020size}) are unfit for our scenario. 
In our setting we need to learn to extrapolate from a causal model.
\vspace{-5pt}
\paragraph{Contributions.}
Our contributions are as follows:
\begin{enumerate}[leftmargin=*]
\vspace{-5pt}
    \item We provide a causal model that formally describes a class of graph classification tasks where the training (${\rm P}^\text{tr}(\cG)$) and test (${\rm P}^\text{te}(\cG)$) graphs have different size and vertex attribute distributions. 
    \item Assuming Independence between Cause and Mechanism (ICM)~\citep{louizos2017causal,shajarisales2015telling}, we introduce a graph representation method based on the work of~\citet{lovasz2006limits} and Graph Neural Networks (GNNs)~\citep{Kipf2016, Hamilton2017, pmlr-v97-you19b}
    that is invariant to the train/test distribution shifts of our causal model. 
    Unlike existing invariant representations, this representation can perform extrapolations from single training environment (e.g., all training graphs have the same size).
    \item Our empirical results show that, in most experiments, neither Invariant Risk Minimization (IRM)~\citep{arjovsky2019invariant} nor the GNN extrapolation modifications proposed by~\citet{xu2020neural} are able to perform well in graph classification tasks over the OOD test data.
\end{enumerate}

\section{Graph Classification: A Causal Model Based on Random Graphs}
\label{sec:family}

\paragraph{Out-of-distribution (OOD) shift.} For any joint distribution ${\rm P}(Y,\cG)$ of graphs $\cG$ and labels $Y$, there are infinitely many causal models that give the same joint distribution~\citep{pearl2009causality}.
This phenomenon is known as model underspecification.
Hence, if the training data distribution ${\rm P}^\text{tr}(Y,\cG)$ does not have the same support as the test distribution ${\rm P}^\text{te}(Y,\cG)$, 
a model trained with samples drawn from ${\rm P}^\text{tr}(Y,\cG)$ needs to be able to extrapolate in order to correctly predict ${\rm P}^\text{te}(Y|\cG)$.
In this work, we assume Independence between Cause and Mechanism (ICM): ${\rm P}^\text{tr}(Y|\cG) = {\rm P}^\text{te}(Y|\cG)$, which is a common assumption in the causal deep learning literature~\citep{bengio2019meta,besserve2018group,johansson2016learning,louizos2017causal,raj2020causal,scholkopf2019causality,arjovsky2019invariant}. 

In inductive graph classification tasks, ICM implies that the shift between train and test distributions ${\rm P}^\text{tr}(Y,\cG) \neq {\rm P}^\text{te}(Y,\cG)$ comes from ${\rm P}^\text{tr}(\cG) \neq {\rm P}^\text{te}(\cG)$, since ${\rm P}^\text{tr}(Y|\cG) = {\rm P}^\text{te}(Y|\cG)$.
And because our task is inductive, i.e., no data from ${\rm P}^\text{te}(\cG)$ or a proxy variable, we must make assumptions about the causal mechanisms in order to extrapolate.

\paragraph{Causal model.}
A graph representation that is robust (invariant) to shifts in ${\rm P}^\text{te}(\cG)$ must know how the distribution shifts.
Either we are given some examples from ${\rm P}^\text{te}(\cG)$ (a.k.a.\ covariate shift adaptation~\citep{sugiyama2007covariate}) or we are given a causal structure that describes how the test distribution can shift.
Our paper focuses on the latter by giving a Structural Causal Model (SCM) for the data generation process \update{in \Cref{def:trainG,def:testG}. The definition of the Structural Causal Model (SCM) is needed since the observational probability itself does not provide any causal information (see observational equivalence in~\citet[Theorem 1.2.8]{pearl2009causality}).}
\Cref{fig:SCM} depicts the Directed Acyclic Graph (DAG) of our causal model. 
It uses the twin network DAGs structure first proposed by~\citet{balke1994twinnets} (see~\citet[Chapter 7.1.4]{pearl2009causality}) in order to define how the test distribution can change.

\newcommand{\supp}{\text{supp}}

In what follows we detail the SCM in \Cref{def:trainG,def:testG}.
Our causal model is inspired by Stochastic Block Models (SBMs)~\citep{diaconis1981statistics,snijders1997estimation} and their connection to graphon random graph models~\citep{airoldi2013stochastic,lovasz2006limits}:
\begin{definition}[Training Graph $\cG^\text{tr}_{N^\text{tr}}$] \label{def:trainG}
The training graph SCM is depicted at the left side of the twin network DAG in \Cref{fig:SCM}.
\begin{itemize}[leftmargin=*]
\item The training graph is characterized by a graphon $W \sim {\rm P}(W)$, where $W:[0,1]^2\rightarrow[0,1]$ is a random symmetric measurable function~\citep{lovasz2006limits} sampled (according to some distribution) from $\sD_W$, the set of all symmetric measurable functions on $[0,1]^2\rightarrow[0,1]$.
$W$ defines both the graph's target label and some of its structural and attribute characteristics, but $W$ is unknown.
\item The {\bf training environment} $E^\text{tr} \sim {\rm P}^\text{tr}(E)$ is a hidden environment variable that represents specific graph properties that change between the training and test.
$E^\text{tr}\in\mathbb{E}$ for some properly defined environment space $\mathbb{E}$.
\item The graph's size is determined by its environment $N^\text{tr} := \eta(E^\text{tr})$, where $\eta$ is an unknown deterministic function.
\item The graph's target label is given by $Y := h(\Wvar,Z_Y)$, $Y\in \sY$, with $\sY$ some properly defined discrete target space. $Z_Y$ is an independent random noise variable and $h$ is a deterministic function on the input space $\sD_W\times \sR$.
\item The vertices are numbered $V^\text{tr}=\{1,\ldots,N^\text{tr}\}$. Each vertex $v \in V^\text{tr}$ has an associated hidden variable $U_v \sim \text{Uniform}(0,1)$ sampled i.i.d..
The graph is undirected and its 
adjacency matrix $A^\text{tr}\in \{0,1\}^{N^\text{tr}\times N^\text{tr}}$ is defined by
\begin{equation} \label{eq:Xtr}
 A_{u,v}^\text{tr} \!:= \mathds{1}(Z_{u,v}\!>\!W(U_u,U_v)), \forall u,v \in V^\text{tr}, u\!\neq\!v.
\end{equation}
The diagonals are set to $0$ because there is no self-loop. Here $\mathds{1}$ is an indicator function, and $\{Z_{u,v}=Z_{v,u}\}_{u,v\in V^\text{tr}}$ are independent uniform noises on $[0,1]$.

\item The graph may contain discrete vertex attributes $X^\text{tr}\in \sX^{N^\text{tr}}$ defined as $$X_{v}^\text{tr} := g_X(E^\text{tr}, W(U_v,U_v)), \:\forall v\in V^\text{tr},$$
where $X_{v}^\text{tr}\in \sX$, and $\sX$ is some properly defined attribute space. $g_X$ is a deterministic function that determines a vertex attribute using $W(U_v,U_v) \in [0,1]$ via, say, inverse sampling~\citep{tweedie1945inverse} the vertex attribute distribution.

\item Then, the training graph is $$\cG^\text{tr}_{N^\text{tr}} :=(A^\text{tr},X^\text{tr}).$$
\end{itemize}
\end{definition}

The test data comes from the following (coupled) distribution, that is, the model uses some of the same random variables of the training graph model, effectively only replacing $E^\text{tr}$ by $E^\text{te}$, as shown in the DAG
of \Cref{fig:SCM}.

\begin{definition}[Test Graph $\cG^\text{te}_{N^\text{te}}$]
\label{def:testG}
The SCM of the test graph is given by the right side of the twin network DAG in \Cref{fig:SCM}, changing the following variables from  \Cref{def:trainG}:
\begin{itemize}[leftmargin=*]
\item The {\bf test environment} $E^\text{te} \sim {\rm P}^\text{te}(E)$, and $E^\text{te}\in \mathbb{E}$ belongs to the same space as $E^\text{tr}$. 
It represents specific properties of the graphs that change between the test and training data. Denote $\supp(\cdot)$ $:= \{x | {\rm P}(x) > 0\}$ as the support of a random variable.
The supports of $E^\text{te}$ and $E^\text{tr}$ may not overlap (i.e., $\supp(E^\text{te}) \cap \supp(E^\text{tr}) = \emptyset$).
\item The change in environment from $E^\text{tr}$ to $E^\text{te}$ may change the graph's size as $N^\text{te} := \eta(E^\text{te})$, where $\eta$ is the same unknown deterministic function as in \Cref{def:trainG}.
\item The vertices are numbered $V^\text{te}=\{1,\ldots,N^\text{te}\}$. The adjacency matrix $A^\text{te}\in \{0,1\}^{N^\text{te}\times N^\text{te}}$ is defined as in \Cref{eq:Xtr}.
\item The graph may contain discrete vertex attributes $X^\text{te}\in\sX^{N^\text{te}}$ defined as 
$$X_{v}^\text{te} := g_X(E^\text{te}, W(U_v,U_v)), \:\forall v \in V^\text{te},$$
with $g_X$ as given in \Cref{def:trainG}.
\item Then, the test graph is
$$\cG^\text{te}_{N^\text{te}} :=(A^\text{te},X^\text{te}).$$
\end{itemize}
\end{definition}

Our SCM has a direct connection with graphon random graph model~\citep{lovasz2006limits}, and extend\update{s} it by considering vertex attributes. Now we introduce examples of our graph classification tasks based on \Cref{def:trainG,def:testG} using two classic random graph models.

{\bf Notation:} ($\cG^\wcard_{N^\wcard}, E^\wcard, A^\wcard, V^\wcard, X^\wcard$) In what follows we use the superscript $\wcard$ as a wildcard to describe both train and test random variables. For instance, $\cG^\wcard_{N^\wcard}$ is a variable that is a wildcard for referring to either $\cG^\text{tr}_{N^\text{tr}}$ or $\cG^\text{te}_{N^\text{te}}$.
Also, from now on we define ${\rm P}^\text{te}(\cG) = {\rm P}(\cG^\text{te}_{N^\text{te}})$ and ${\rm P}^\text{tr}(\cG) = {\rm P}(\cG^\text{tr}_{N^\text{tr}})$.

\paragraph{\erdosrenyi example.}
Consider a random training environment $E^\text{tr}$ such that $N^\text{tr} = \eta (E^\text{tr})$ is the number of vertices for graphs in our training data. Let $p$ be the probability that any two distinct vertices of the graph have an edge. Define $W$ as a constant function that always outputs $p$.
Sample independent uniform noises $Z_{u,v} \sim \text{Uniform}(0,1)$ (for each possible edge, $Z_{u,v}=Z_{v,u}$).
An \erdosrenyi graph can be defined as a graph whose adjacency matrix $A^\text{tr}$ is $A^\text{tr}_{u,v} = \mathds{1}(Z_{u,v}>W(U_u,U_v))=\mathds{1}(Z_{u,v}>p)$, $\forall u,v \in V^\text{tr},u\neq v$. Here vertex attributes are not considered and \update{we} can define $X^\text{tr}_v=\O, \forall v\in V^{\text{tr}}$ as the null attribute. 

In the test data, we have a different environment $E^\text{te}$ and graph size $N^\text{te} = \eta (E^\text{te})$, with $\supp(N^\text{te}) \cap \supp(N^\text{tr}) = \emptyset$.
The variable $\{Z_{u,v}\}_{u,v\in \{1,\ldots,\max(\supp(N^\text{tr}) \cup \supp(N^\text{te}))\}}$ 
can be thought as the seed of a random number generator 
to determine if two distinct vertices $u$ and $v$
are connected by an edge.
The above defines our training and test data as a set of \erdosrenyi random graphs of sizes $N^\text{tr}$ and $N^\text{te}$ with probability
$p$.
The targets of the \erdosrenyi graphs can be, for instance, the value $Y = p$ in \Cref{def:trainG}, which is determined by $W$ and invariant to graph sizes.

\paragraph{Stochastic Block Model (SBM)~\citep{snijders1997estimation}.} An SBM can be seen as a generalization of \erdosrenyi graphs.
SBMs partition the vertex set into disjoint subsets $S_1,S_2,...,S_r$ (known as blocks or communities) with an associated $r\times r$ symmetric matrix $\mP$, where the probability of an edge $(u,v)$, $u\in S_i$ and $v\in S_j$ is $\mP_{ij}$, for $i,j\in \{1,\ldots,r\}$. 
In the training and test data, we still have i.i.d sampled $Z_{u,v}=Z_{v,u}$ and different environments $E^\text{tr}$, $E^\text{te}$. 
Divide the interval $[0,1]$ into disjoint convex sets $[t_0,t_1), [t_1,t_2),\ldots,[t_{r-1},t_r]$, where $t_0=0$ and $t_r=1$, such that if $U_v \sim \text{Uniform}(0,1)$ satisfies $U_v \in [t_{i-1},t_i)$, then vertex $v$ belongs to block $S_i$. Thus $W(U_u,U_v)=\sum_{i,j\in \{1,\ldots,r\}}P_{ij}\mathds{1}(U_u \in [t_{i-1},t_i))\mathds{1}(U_v \in [t_{j-1},t_j))$. An SBM graph in training or test can be defined as a graph whose adjacency matrix $A^\wcard$ is $A^\wcard_{u,v} = \mathds{1}(Z_{u,v} > W(U_u,U_v))$, $\forall u,v \in V^\wcard, u\neq v$. Now we have a set of SBM random graphs of sizes $N^\text{tr}$ and $N^\text{te}$ with $\mP$. Consider if there are only two blocks, the target $Y$ can be $\mP_{1,2}$ which is the probability of an edge connecting vertices between the blocks, determined by $W$ and invariant to graph sizes.

\paragraph{SBM with vertex attributes.} For the SBM, assume the vertex attributes are tied to blocks, and are distinct for each block.
The environment variable operates on changing
the distributions of attributes assigned in each block. 
Consider the following SBM example with two blocks:
Define $W(U_v,U_v)=\frac{U_v}{2t_1}\mathds{1}(U_v\in [0,t_1))+ (\frac{1}{2}+\frac{U_v-t_1}{2(1-t_1)})\mathds{1}(U_v\in [t_1,1])$. So $W(U_v,U_v)< \frac{1}{2}$ if and only if $v$ belongs to the first block. We only change the values of $W$ for points on a zero-measure space.
Let $g_X$ be such that it defines constants \update{as} $0 < \alpha_{E^\wcard\!,1} < \frac{1}{2} < \alpha_{E^\wcard\!,2} < 1$, and 
vertex attributes as
$$X^\wcard_v\! =\! g_X(E^\wcard \!,W(U_v,U_v))\!=\!\negthickspace
\begin{bmatrix*}[l]
\mathds{1}(W(U_v,U_v) \! \in \! [0,\alpha_{E^\wcard\!,1}))\\
\mathds{1}(W(U_v,U_v) \! \in \! [\alpha_{E^\wcard\!,1},.5))\\
\mathds{1}(W(U_v,U_v) \! \in \! [.5,\alpha_{E^\wcard\!,2}))\\
\mathds{1}(W(U_v,U_v) \! \in \! [\alpha_{E^\wcard\!,2},1])
\end{bmatrix*}\negthickspace,$$
where the attribute of vertex $v$, $X^\wcard_v$, is one-hot encoded to represent 4 colors: red and blue (if $v$ is in block $1$) and green and yellow (if $v$ is in block 2). 
\section{E-Invariant Graph Representations}
\label{sec:methods}
In this section we discuss shortcomings of traditional graph representation methods for out-of-distribution (OOD) graph classification tasks. We will base our discussion on our Structural Causal Model (SCM) (described in \Cref{def:testG,def:trainG} and \Cref{fig:SCM}).
We show that there is an approximately environment-invariant graph representation that is able to extrapolate to OOD test data.

\paragraph{The shortcomings of standard graph representation methods.}
\Cref{fig:SCM} shows that our target variable $Y$ is a function only of the {\em graphon} variable $\Wvar$, rather than the training or test environments, $E^\text{tr}$ and $E^\text{te}$, respectively.
However, $Y$ is not independent of $E^\text{tr}$ given $\cG^\text{tr}_{N^\text{tr}}$, since both $E^\text{tr}$ and $W$ affect $A^\text{tr}$ and $X^\text{tr}$ (which are colliders), and $Y$ depends on $W$.
Hence, traditional graph representation learning methods can pick up this easy spurious correlation in the training data (via shortcut learning~\citep{geirhos2020shortcut}), which would prevent the model learning the correct OOD test predictor. 

\update{To address the challenge of correctly predicting $Y$ in our OOD test data, regardless of spurious correlations between the variables, we need an estimator that can account for it.
In what follows we focus on {\bf environment-invariant (E-invariant)} graph representations. 
To show the ability of E-invariant representations to extrapolate to OOD test data, we introduce the definition and the effect on downstream OOD classification tasks in the following proposition.} 

\begin{restatable}{proposition}{propEinv}[E-invariant Representation's Effect on \update{OOD} Classification]\label{prop:Einv}
Consider a permutation-invariant graph representation $\Gamma:\cup_{n=1}^\infty \{0,1\}^{n\times n} \times \sX^n \to \sR^d$, $d \geq 1$, and a downstream function $\rho:\sY \times \sR^d \to [0,1]$ (e.g., a feedforward neural network (MLP) with softmax outputs) such that, for some $\epsilon, \delta > 0$, the generalization error over the training distribution is: $\forall y \in \sY$,
\[
{\rm P}(~\vert {\rm P}(Y = y | \cG^\text{tr}_{N^\text{tr}}) - \rho(y,\Gamma(\cG^\text{tr}_{N^\text{tr}})) \vert~ \leq \epsilon) \geq 1 - \delta ,
\] 
$\Gamma$ is said to be {\bf environment-invariant (E-invariant)} if $\forall e \in \supp(E^\text{tr}),\forall e^\dagger \in \supp(E^\text{te})$,
\[
\Gamma(\cG^\text{tr}_{N^\text{tr}}|E^\text{tr}=e)=\Gamma(\cG^\text{te}_{N^\text{te}}|E^\text{te}=e^\dagger).
\]
\update{If $\Gamma$ is E-invariant}, then the OOD test error is the same as the generalization error over the training distribution, i.e., $\forall y \in \sY$,
\begin{equation} \label{eq:predY_Einv}
    {\rm P}(\vert {\rm P}(Y = y| \cG^\text{te}_{N^\text{te}}) -  \rho(y,\Gamma(\cG^\text{te}_{N^\text{te}}))\vert \leq \epsilon) \geq 1 - \delta.
\end{equation}
\end{restatable}
\Cref{prop:Einv} shows that an E-invariant representation will perform no worse on the OOD test data (extrapolation
samples from $(Y,\cG^\text{te}_{N^\text{te}})$) than on a test dataset having the same environment distribution as the training data (samples from $(Y,\cG^\text{tr}_{N^\text{tr}})$). {\em Our task now becomes finding an E-invariant graph representation $\Gamma$ that can be used to predict $Y$.}

\paragraph{The shortcomings of Invariant Risk Minimization (IRM).}
Invariant Risk Minimization (IRM)~\citep{arjovsky2019invariant} aims to learn a representation that is invariant across all training environment\update{s}, $\forall e \in \supp(E^\text{tr})$, by adding a regularization penalty on the empirical risk. However, IRM will fail if: (i) $\supp(E^\text{te}) \not \subseteq \supp(E^\text{tr})$, since the penalty provides no guarantee that the representation will still be invariant w.r.t.\ $e^\dagger \in \supp(E^\text{te})\backslash \supp(E^\text{tr})$ if the representation is a nonlinear function of the input~\citep{rosenfeld2020risks}; and
(ii) if the training data only contains a single environment, i.e., $\supp(E^\text{tr})=\{e\}$. For instance, the training data may contain only graphs of a single size. In this case, we are unable to apply IRM for size extrapolations.
Our experiments show that the IRM procedure 
does not seem to work for graph representation learning.

In what follows we leverage the stability of subgraph densities (more precisely, induced homomorphism densities) in graphon random graph models~\citep{lovasz2006limits} to learn E-invariant representations for the SCM defined in \Cref{def:testG,def:trainG}, whose DAG is illustrated in \Cref{fig:SCM}.

\subsection{An Approximately E-Invariant Graph Representations for Our Model}\label{sec:intuitive}

Let $\cG^\wcard_{N^\wcard}$ denote either an $N^\text{tr}$-sized train or $N^\text{te}$-sized test graph from the SCM in \Cref{def:trainG,def:testG}.
For a given $k$-vertex graph $F_k$ $(k<N^\wcard)$, let $\text{ind}(F_k,\cG^\wcard_{N^\wcard})$ be the number of induced homomorphisms of $F_k$ into $\cG^\wcard_{N^\wcard}$, informally, the number of mappings from $V(F_k)$ to $V(\cG^\wcard_{N^\wcard})$ such that  the corresponding subgraph induced in $\cG^\wcard_{N^\wcard}$ is isomorphic to $F_k$. The induced homomorphism density is defined as
\begin{equation}\label{eq:tinj}
    t_{\rm ind}(F_k,\cG^\wcard_{N^\wcard})=\frac{{\rm ind}(F_k,\cG^\wcard_{N^\wcard})}{N^\wcard!/(N^\wcard - k)!},
\end{equation}
where the denominator is the number of possible mappings. Let $\allpossibleGraphsK$ be the set of all connected vertex-attributed graphs of size $k' \leq k$.
Using the subgraph densities (induced homomorphism densities) $\{t_{\rm ind}(F_{k'},\cG^\wcard_{N^\wcard})\}_{F_{k'} \in \cF_{\leq k}}$ we will construct a (feature vector) representation for $\cG^\wcard_{N^\wcard}$, similar to~\citet{hancock2020survey,pinar2017escape},
  \begin{equation}
  \label{eq:Gamma-1hot}
      \Gamma_\text{1-hot}(\cG^\wcard_{N^\wcard})\! =\!\!\!\!\!\!\! \sum_{\subgraph\in \allpossibleGraphsK} \!\! t_\text{ind}(\subgraph, \cG^\wcard_{N^\wcard}) \vone_\text{one-hot}\{\subgraph, \allpossibleGraphsK\},
  \end{equation}%
 where $\vone_\text{one-hot}\{\subgraph, \allpossibleGraphsK\}$ assigns a unique one-hot vector to each distinct graph $\subgraph$ in $\allpossibleGraphsK$. For instance, for $k=4$, the one-hot vectors could be (1,0,\ldots,0)=\VeeNbrg, (0,1,\ldots,0)=\Kfourbrgg, (0,0,\ldots,1,\ldots,0)=\VeeNbrr, (0,0,\ldots,1)=\TrigNbrg, etc..
In \Cref{sec:conditions} we show that the (feature vector) representation in \Cref{eq:Gamma-1hot} is approximately environment-invariant in our SCM model.
  
An alternative approach is to replace the one-hot vector representation with learnable graph representation models. We first use Graph Neural Networks (GNNs)~\citep{Kipf2016, Hamilton2017, pmlr-v97-you19b} to learn representations that can capture information from vertex attributes. 
Simply speaking, GNNs proceed by vertices passing messages, amongst each other, through a learnable function such as an MLP, and repeating $L \in \mathbb{Z}_{\ge 1}$ layers.

Consider the following simple GNN example. Let $V^\wcard$ be the set of vertices. At each iteration $l \in \{1, 2, \ldots, L\}$, all vertices $v\in V^\wcard$ are associated with a learned vector $\vh_v^{(l)}$. Specifically, we begin by initializing a vector as $\vh_{v}^{(0)} = X_{v}$ for every vertex $v \in V^\wcard$. Then, we recursively compute an update such as the following $\forall v \in V^\wcard$,
\begin{equation}\label{eq:WL1}%
\vh^{(l)}_v \!=\mathrm{MLP}^{(l)} \! \Big(\vh^{(l-1)}_v\!, \text{READOUT}_\text{Neigh}((\vh^{(l-1)}_u)_{u \in \mathcal{N}\!(v)}) \!\Big),
\end{equation}%
where $\mathcal{N}(v) \subseteq V^\wcard$ denotes the neighborhood set of $v$ in the graph, $\text{READOUT}_\text{Neigh}$ is a permutation-invariant function (e.g. sum) of the neighborhood learned vectors, $\mathrm{MLP}^{(l)}$ denotes a multi-layer perceptron and whose superscript $l$ indicates that the MLP at each recursion layer may have different learnable parameters.
There are other alternatives to \Cref{eq:WL1} that we will also test in our experiments.

Then, we arrive to the following
representation of $\cG^\wcard_{N^\wcard}$:
\begin{equation}
\label{eq:Gamma-kgnn}
\begin{aligned}[b]
    &\Gamma_\text{GNN}(\cG^\wcard_{N^\wcard} ) =\\
    &\sum_{\subgraph\in \allpossibleGraphsK} t_\text{ind}(\subgraph,\cG^\wcard_{N^\wcard}) \text{READOUT}_\Gamma(\text{GNN}(\subgraph)),
\end{aligned}
\end{equation}
where
$\text{READOUT}_\Gamma$ is a permutation-invariant function that maps the vertex-level outputs of a GNN to a graph-level representation (e.g. by summing all vertex embeddings). Unfortunately, GNNs are not most-expressive representations of graphs~\citep{morris2019weisfeiler,pmlr-v97-murphy19a,xu2018powerful} and thus $\Gamma_\text{GNN}(\cdot)$ is less expressive than $\Gamma_\text{1-hot}(\cdot)$.  A representation with greater expressive power is
\begin{equation} \label{eq:Gamma-krpgnn}
\begin{aligned}[b]
    &\Gamma_{\text{GNN}^+}(\cG^\wcard_{N^\wcard})=\\
    &\sum_{\subgraph\in \allpossibleGraphsK} t_\text{ind}(\subgraph, \cG^\wcard_{N^\wcard} ) \text{READOUT}_\Gamma(\text{GNN}^+(\subgraph)),
\end{aligned}
\end{equation}
where $\text{GNN}^+$ is a most-expressive $k'$-vertex graph representation, which can be achieved by any of the methods of~\citet{vignac2020building, maron2019provably,pmlr-v97-murphy19a}.
Since $\text{GNN}^+$ is most expressive, $\text{GNN}^+$ can ignore attributes and map each $\subgraph$ to a one-hot vector $\vone_\text{one-hot}\{\subgraph, \allpossibleGraphsK\}$; therefore, $\Gamma_{\text{GNN}^+}(\cdot)$ generalizes $\Gamma_\text{1-hot}(\cdot)$ of \Cref{eq:Gamma-1hot}. But {\em note that greater expressiveness does not imply better extrapolation}.
 
More importantly, GNN and $\text{GNN}^+$ representations allow us to increase their E-invariance by adding a penalty for having different representations of two graphs $F_{k'}$ and $H_{k'}$ with the same topology but different vertex attributes (say, $F_{k'}\!\!=$~\Kfourbrgg and $H_{k'}\!\!=$~\Kfourbrgr), as long as these differences do not significantly impact downstream model accuracy in the training data.
\update{Note that this is more powerful than simply masking vertex attributes, since it allows same-topology graphs with distinct vertex attributes to have different representations if it is important to distinguish them for the target prediction (see \Cref{sec:attribute}).}
We will discuss more about these theoretical underpinnings in the next section.
Hence, for each $k'$-sized vertex-attributed graph $F_{k'}$, we consider the set \(\samestructure\) of all $k'$-sized vertex-attributed graphs having the same underlying topology as $F_{k'}$ but with all possible different vertex attributes.
We then define the regularization penalty
\begin{align}
    \frac{1}{|\allpossibleGraphsK|}&\sum_{\subgraph\in \allpossibleGraphsK}
    \mathbb{E}_{H_{k'} \in \samestructure} \bigl[ \Vert \text{READOUT}_\Gamma(\text{GNN}^{*}(\subgraph)) \nonumber \\
    &- \text{READOUT}_\Gamma(\text{GNN}^{*}(H_{k'})) \Vert_2 \bigr] \label{eq:regul},
\end{align}
where $\text{GNN}^{*} = \text{GNN}$ if we choose the representation $\Gamma_\text{GNN}$, or $\text{GNN}^{*} = \text{GNN}^{+}$ if we choose the representation $\Gamma_{\text{GNN}^+}$. 
In practice, we assume $H_{k'}$ is uniformly sampled from $\samestructure$ and we sample one $H_{k'}$ for each $F_{k'}$
in order to obtain an unbiased estimator of \Cref{eq:regul}.

\paragraph{Practical considerations.}
\update{Efficient algorithms exist to obtain {\em induced} homomorphism densities over all possible {\em connected} $k$-vertex subgraphs~\citep{ahmed2016estimation, bressan2017counting, chen2018mining, chen2016general, rossi2019heterogeneous, wang2014efficiently}.
For unattributed graphs and $k \leq 5$, we use ESCAPE~\citep{pinar2017escape} to obtain {\em exact} densities. For attributed graphs or unattributed graphs with $k > 5$, exact counting becomes intractable, so we use R-GPM~\citep{teixeira2018graph} to obtain unbiased estimates of densities. Finally, \Cref{prop:bias} in \Cref{sec:BiasAppendix} shows that certain biased estimators can also be used if $\text{READOUT}_\Gamma$ is the sum of vertex embeddings.}

\subsection{Theoretical Description of our E-Invariant Graph Representations} 
\label{sec:conditions}

In this section, we show that the graph representations seen in the previous section are approximately environment-invariant in our SCM model under mild assumptions.

\begin{restatable}[Approximate\update{ly} E-invariant Graph Representation]{theorem}{thmsizeExtrapolationBound}
\label{thm:sizeExtrapolationBound}

Let $\cG^\text{tr}_{N^\text{tr}}$ and $\cG^\text{te}_{N^\text{te}}$ be two samples of graphs of sizes ${N^\text{tr}}$ and ${N^\text{te}}$ from the training and test distributions, respectively, both defined over the same graphon variable $W$ and satisfying \Cref{def:trainG,def:testG}.
Assume the vertex attribute function $g_X(\cdot,\cdot)$ of \Cref{def:trainG,def:testG} is invariant to $E^\text{tr}$ and $E^\text{te}$ (the reason for this assumption will be clear later).
Let $||\cdot||_\infty$ denote the $L$-infinity norm.
For any integer $k\leq \min(N^\text{tr},N^\text{te})$, and any constant $0<\epsilon<1$,
\begin{equation}
\begin{aligned}[b]
    {\rm P}(\Vert\Gamma_\text{1-hot}&(\cG^\text{tr}_{N^\text{tr}}) -\Gamma_\text{1-hot}(\cG^\text{te}_{N^\text{te}})\Vert_\infty>\epsilon)\leq \\ &2|\allpossibleGraphsK|(\exp(-\frac{\epsilon^2 N^\text{tr}}{8k^2})+\exp(-\frac{\epsilon^2 N^\text{te}}{8k^2})).
\end{aligned}
\end{equation}
\end{restatable}
\Cref{thm:sizeExtrapolationBound} shows how the graph representations given in \Cref{eq:Gamma-1hot} are approximately E-invariant. 
Note that for unattributed graphs, we can define $g_X(\cdot,\cdot)=\O$ as the null attribute, which is invariant to any environment by construction. For graphs with attributed vertices, $g_X(\cdot,\cdot)$ being invariant to $E^\text{tr}$ and $E^\text{te}$ means \update{that} for any two environments $e\in \supp(E^\text{tr}), e^\dagger \in \supp(E^\text{te})$, $g_X(e,\cdot)=g_X(e^\dagger,\cdot)$. 

\Cref{thm:sizeExtrapolationBound} shows that for $k \ll \min({N^\text{tr}},{N^\text{te}})$, the representations $\Gamma_\text{1-hot}(\cdot)$ of two possibly different-sized graphs with the same $W$ are nearly identical, \update{indicating}
$\Gamma_\text{1-hot}(\cG^\wcard_{N^\wcard})$ is an approximately E-invariant representation.

\Cref{thm:sizeExtrapolationBound} also exposes a trade-off, however. 
If the observed graphs tend to be relatively small, the required $k$ for approximately E-invariant representations can be small, and then the expressiveness of $\Gamma_\text{1-hot}(\cdot)$ gets compromised.
That is, the ability of $\Gamma_\text{1-hot}(\cG^\wcard_{N^\wcard})$ to extract information about $W$ from $\cG^\wcard_{N^\wcard}$ reduces as $k$ decreases. Finally, this guarantees that for appropriate $k$, passing the representation $\Gamma_\text{1-hot}(\cG^\wcard_{N^\wcard})$ to a downstream classifier provably approximates the classifier in \Cref{eq:predY_Einv} of \Cref{prop:Einv}.

Note that when the vertex attributes are not invariant to the environment variable, $\Gamma_\text{1-hot}(\cdot)$ is not E-invariant and we can not extrapolate using $\Gamma_\text{1-hot}(\cdot)$.
Thankfully, for the GNN-based graph representations $\Gamma_\text{GNN}(\cG^\wcard_{N^\wcard})$ and $\Gamma_{\text{GNN}^+}(\cG^\wcard_{N^\wcard})$ in \Cref{eq:Gamma-kgnn,eq:Gamma-krpgnn}, respectively, the regularization penalty in \Cref{eq:regul} pushes the graph representation 
to be more E-invariant, making it more likely to satisfy the conditions of E-invariance in \Cref{thm:sizeExtrapolationBound}.
\Cref{eq:regul} is inspired by the {\em asymmetry learning} procedure of \citet{mouli2021neural}, which induces symmetry priors in the neural network, which can be broken (making the neural network asymmetric) only when imposing the symmetry significantly increases the training loss.

To understand the effect of our {\em asymmetry learning} in regularizing towards topology, consider the attributed SBM example in \Cref{sec:family}.
The environment operates by changing the distributions of attributes assigned within each block. If we are going to achieve E-invariance (and correctly predict cross-block edge probabilities in the test data (see \Cref{sec:attribute})), we need graph representations that \update{treat} attributes assigned to the same block as equivalent.
By regularizing the GNN-based graph representations towards focusing only on topology rather than vertex attributes, the regularization forces the GNN to treat all within-block vertex attributes as equivalent, and achieve an approximately E-invariant representation in this setting.
And since treating the across-block vertex attributes as equivalent hurts the training loss in this setting, these will not be considered equivalent by the GNN.
%
%
\section{Related Work}\label{sec:relatedwork}

%
This section presents an overview of the related work. Due to space constraints, a more in-depth discussion with further references is given in \Cref{sec:RelatedWorkAppendix}.

\vspace{-5pt}
\paragraph{OOD extrapolation in graph classification and size extrapolation in GNNs.}
Our work ascertain\update{s} a causal relationship between graphs and their target labels. We are unaware of existing work on this topic.
\citet{xu2020neural} is interested on a geometric (non-causal) definition of extrapolation for a class of graph algorithms.
\citet{hu2020open} introduces a large graph dataset presenting significant challenges of OOD extrapolation, however, their shift is on the two-dimensional structural framework distribution of the molecules, and no causal model is provided.
The parallel work of~\citet{yehudai2020size} improves size extrapolation in GNNs using self-supervised and semi-supervised learning on both the training and test domain, which is orthogonal to our problem. 
Previous works also examine empirically the ability of graph neural networks to extrapolate in various applications, such as physics~\citep{battaglia2016interaction,sanchez2018graph}, mathematical and abstract reasoning~\citep{santoro2018measuring,saxton2018analysing}, and graph algorithms~\citep{bello2016neural,nowak2017note,battaglia2018relational,joshi2020learning, velivckovic2019neural, hao2020towards}.
These works do not provide guarantees of test extrapolation performance, a causal model, or a proof that the tasks require extrapolation over different environments.

\vspace{-5pt}
\paragraph{Causal reasoning and invariances.}
Recent efforts have brought counterfactual inference to machine learning models, including {\em Independence of causal mechanism (ICM)} methods~\citep{bengio2019meta,besserve2018group,johansson2016learning,louizos2017causal,parascandolo2018learning,raj2020causal,scholkopf2019causality}, {\em Causal Discovery from Change (CDC)} methods~\citep{tian2001causal}, and {\em representation disentanglement} methods~\citep{bengio2019meta,goudet2017causal,locatello2019challenging}. 
Invariant risk minimization (IRM)~\citep{arjovsky2019invariant} is a type of ICM~\citep{scholkopf2019causality}. \update{Risk
Extrapolation (REx)~\citep{krueger2021outofdistribution} optimizes by focusing on the training environments that have the largest impact on training.}

Broadly, the above efforts look for representations (or mechanism descriptions) that are invariant across multiple environments observed in the training data.
In our work, we are interested in techniques that can work with a single training environment \update{and when the test support is not a subset of the train support} --- a common case in graph data.
To the best of our knowledge, the only representation learning work considering single environment extrapolations is \citet{mouli2021neural}.
However, none of these methods is specifically designed for graphs, and it is unclear how they can be efficiently adapted for graph tasks.
Finally, we also note that domain adaptation techniques and recent work on domain-predictors~\citep{chuang2020estimating} aim to learn invariances that can be used for the predictions. However, these require access to test data during training, which is not our scenario.

\vspace{-5pt}
\paragraph{Graph classification using induced homomorphisms.} 
A related set of works looks at induced homomorphism densities as graph features for a kernel~\citep{shervashidze2009efficient,yanardag2015deep,wale2008comparison}. These methods can perform poorly in some tasks~\citep{kriege2018property}. \update{Recent work has also shown an interest in induced subgraphs, which are used to improve predictions of GNNs~\citep{bouritsas2020improving} or treated as inputs for newly-proposed architectures~\citep{toenshoff2021graph}. Also note that the graph representations $\Gamma_\text{GNN}(\cdot)$ and $\Gamma_{\text{GNN}^+}(\cdot)$ in \Cref{eq:Gamma-kgnn,eq:Gamma-krpgnn} respectively,
have similarities to $k$-ary Relational Pooling~\citep{pmlr-v97-murphy19a} with the main difference being that the subgraph representations are weighted in our case.}
None of these methods focus on invariant representations or extrapolations.

\vspace{-5pt}
\paragraph{Expressiveness of graph representations.}
The expressiveness of a graph representation method is a measure of model family bias~\citep{morris2019weisfeiler,xu2018powerful,gartner2003graph,maron2019provably,pmlr-v97-murphy19a}.
That is, given enough training data, a neural network from a more expressive family can achieve smaller generalization error over the training distribution than a neural network from a less expressive family, assuming appropriate optimization.
However, this power is a measure of generalization capability over the training distribution, not OOD extrapolation.
Hence, the question of representation expressiveness is orthogonal to our work.

\begin{table*}[t]
	\caption{\small Extrapolation performance over {\em unattributed} graphs {\bf shows clear advantage of our environment-invariant representations, with or without GNN, over standard methods or IRM in extrapolation test accuracy}. Table shows mean (standard deviation) accuracy. Bold emphasises the best test average. NA value indicates IRM is not applicable (when training data has a single graph size).}
	\label{tab:unattributed}
	\begin{small}
	\begin{sc}
	\begin{center}
	\resizebox{\textwidth}{!}{
		\begin{tabular}{lrrr||rrr|rrr}
			& \multicolumn{3}{c}{Accuracy in Schizophrenia Task}
			&
			\multicolumn{6}{c}{Accuracy in \erdosrenyi Task}
			\\
			\cmidrule(lr){2-4} \cmidrule(lr){5-10}
		    &
		    \multicolumn{3}{c}{\textbf{Training has a single graph size}}
		    &
			\multicolumn{3}{c}{\textbf{Training has a single graph size}}
			&
			\multicolumn{3}{c}{\textbf{Training has two graph sizes}} \\
			\cmidrule(lr){2-4} \cmidrule(lr){5-7} \cmidrule(lr){8-10}
		 &
		 \multicolumn{1}{c}{Train [$P(Y,\cG^\text{tr}_{N^\text{tr}})$]} & \multicolumn{1}{c}{Val.\ [$P(Y,\cG^\text{tr}_{N^\text{tr}})$]} & \multicolumn{1}{c}{\textbf{Test} ($\uparrow$) [$P(Y,\cG^\text{te}_{N^\text{te}})$]}
		 & 
		 \multicolumn{1}{c}{Train [$P(Y,\cG^\text{tr}_{N^\text{tr}})$]} & \multicolumn{1}{c}{Val.\ [$P(Y,\cG^\text{tr}_{N^\text{tr}})$]} & \multicolumn{1}{c}{\textbf{Test} ($\uparrow$) [$P(Y,\cG^\text{te}_{N^\text{te}})$]}
		 &
		 \multicolumn{1}{c}{Train [$P(Y,\cG^\text{tr}_{N^\text{tr}})$]} & \multicolumn{1}{c}{Val.\ [$P(Y,\cG^\text{tr}_{N^\text{tr}})$]} & \multicolumn{1}{c}{\textbf{Test} ($\uparrow$) [$P(Y,\cG^\text{te}_{N^\text{te}})$]}
		 \\
			 \cmidrule(lr){2-4} \cmidrule(lr){5-7} \cmidrule(lr){8-10}
			    PNA & 0.99 (0.00) & 0.76 (0.08) & 0.61 (0.08)
			        & 1.00 (0.00) & 1.00 (0.00) & 0.65 (0.12)
			        & 1.00 (0.00) & 1.00 (0.00) & 0.64 (0.12) \\
                PNA (mean xu-readout) & 0.99 (0.00) & 0.77 (0.07) & 0.53 (0.10)
                                & 1.00 (0.00) & 1.00 (0.00) & 0.62 (0.12)
                                & 1.00 (0.00) & 1.00 (0.00) & 0.51 (0.19) \\
                PNA (max xu-readout) & 0.99 (0.00) & 0.75 (0.07) & 0.42 (0.06)
                                & 1.00 (0.00) & 1.00 (0.00) & 0.59 (0.16)
                                & 0.99 (0.01) & 1.00 (0.00) & 0.57 (0.15) \\
                PNA + IRM & NA & NA & NA
                & NA & NA & NA
                & 1.00 (0.00) & 1.00 (0.00) & 0.65 (0.13) \\
			 	GCN & 0.74 (0.04) & 0.74 (0.08) & 0.55 (0.09)
			 	    & 0.99 (0.01) & 1.00 (0.00) & 0.88 (0.10)
			 	    & 0.98 (0.01) & 1.00 (0.00) & 0.87 (0.10) \\
			 	GCN (mean xu-readout) & 0.72 (0.04) & 0.73 (0.08) & 0.65 (0.08)
			 	                & 0.99 (0.01) & 1.00 (0.00) & 0.79 (0.15)
			 	                & 0.98 (0.02) & 1.00 (0.00) & 0.75 (0.20) \\
			 	GCN (max xu-readout) & 0.86 (0.07) & 0.75 (0.07) & 0.54 (0.06)
			 	                & 0.99 (0.01) & 1.00 (0.00) & 0.90 (0.07)
			 	                & 0.96 (0.04) & 1.00 (0.00) & 0.87 (0.09) \\
			 	 GCN + IRM  & NA & NA & NA
			 	            & NA & NA & NA
                            & 0.98 (0.02) & 1.00 (0.00) & 0.88 (0.08) \\
				GIN  & 0.72 (0.02) & 0.74 (0.05) & 0.36 (0.09)
				    & 1.00 (0.00) & 1.00 (0.00) & 0.64 (0.12)
				    & 1.00 (0.00) & 1.00 (0.00) & 0.65 (0.12) \\
				GIN (mean xu-readout) & 0.78 (0.02) & 0.72 (0.05) & 0.43 (0.05)
				                & 1.00 (0.00) & 1.00 (0.00) & 0.63 (0.09)
				                & 1.00 (0.00) & 1.00 (0.00) & 0.61 (0.09) \\
			 	GIN (max xu-readout) & 0.85 (0.02) & 0.72 (0.05)  & 0.35 (0.06)
			 	                & 0.99 (0.01) & 1.00 (0.00) & 0.65 (0.12)
			 	                & 1.00 (0.00) & 1.00 (0.00) & 0.65 (0.07) \\
				GIN + IRM  & NA & NA & NA
				            & NA & NA & NA
				            & 1.00 (0.00) & 1.00 (0.00) & 0.66 (0.08) \\
                RPGIN & 0.70 (0.02) & 0.74 (0.05) & 0.37 (0.06)
                        & 1.00 (0.00) & 1.00 (0.00) & 0.61 (0.16)
                        & 1.00 (0.00) & 1.00 (0.00) & 0.60 (0.16) \\
                \hline
				WL Kernel  & 1.00 (0.00) & 0.63 (0.07) & 0.40 (0.00)
				            & 1.00 (0.00) & 1.00 (0.00) & 0.01 (0.00)
				            & 1.00 (0.00) & 1.00 (0.00) & 0.30 (0.00) \\
				GC Kernel  & 0.61 (0.00) & 0.61 (0.06) & 0.60 (0.00)
				            & 1.00 (0.00) & 1.00 (0.00) & \textbf{1.00 (0.00)}
				            & 1.00 (0.00) & 1.00 (0.00) & \textbf{1.00 (0.00)} \\
				\hline
				$\Gamma_\text{1-hot}$ & 0.71 (0.01) & 0.72 (0.05) & \textbf{0.72 (0.04)} 
				        & 1.00 (0.00) & 1.00 (0.00) & \textbf{1.00 (0.00)}
				        & 1.00 (0.00) & 1.00 (0.00) & \textbf{1.00 (0.00)} \\
				$\Gamma_\text{GIN}$  & 0.75 (0.05) & 0.70 (0.04) & \textbf{0.68 (0.07)}
				        & 1.00 (0.00) & 1.00 (0.00) & \textbf{1.00 (0.00)}
				        & 1.00 (0.00) & 1.00 (0.00) & \textbf{1.00 (0.00)} \\
				$\Gamma_\text{RPGIN}$  & 0.69 (0.01) & 0.71 (0.06) & \textbf{0.71 (0.03)}
				        & 1.00 (0.00) & 1.00 (0.00) & \textbf{1.00 (0.00)}
				        & 1.00 (0.00) & 1.00 (0.00) & \textbf{1.00 (0.00)} \\
				\bottomrule
		\end{tabular}
		}
	\end{center}
    \end{sc}
    \end{small}
\vspace{-15pt}
\end{table*}

\begin{table*}[t]
	\caption{\small  Extrapolation performance over {\em attributed} graphs {\bf shows clear advantage of environment-invariant representations with GNNs and the attribute regularization in \Cref{eq:regul}}.
	Table shows mean (standard deviation) accuracy. Bold emphasises the best test average. NA value indicates IRM is not applicable (when training data has a single graph size).}
	\label{tab:attr}
    \begin{small}
	\begin{sc}
	\begin{center}
	\resizebox{\textwidth}{!}{
		\begin{tabular}{lrrr|rrr|rrr}
		    &
			\multicolumn{3}{c}{\textbf{Training has a single graph size 20}}
			&
			\multicolumn{3}{c}{\textbf{Training has two graph sizes: 14 and 20}}
			&
			\multicolumn{3}{c}{\textbf{Training has two graph sizes: 20 and 30}} \\
			\cmidrule(lr){2-4} \cmidrule(lr){5-7} \cmidrule(lr){8-10}
		 & \multicolumn{1}{c}{Train [$P(Y,\cG^\text{tr}_{N^\text{tr}})$]} & \multicolumn{1}{c}{Val.\ [$P(Y,\cG^\text{tr}_{N^\text{tr}})$]} & \multicolumn{1}{c}{\textbf{Test} ($\uparrow$) [$P(Y,\cG^\text{te}_{N^\text{te}})$]}
		 &
		 \multicolumn{1}{c}{Train [$P(Y,\cG^\text{tr}_{N^\text{tr}})$]} & \multicolumn{1}{c}{Val.\ [$P(Y,\cG^\text{tr}_{N^\text{tr}})$]} & \multicolumn{1}{c}{\textbf{Test} ($\uparrow$) [$P(Y,\cG^\text{te}_{N^\text{te}})$]}
		 &
		 \multicolumn{1}{c}{Train [$P(Y,\cG^\text{tr}_{N^\text{tr}})$]} & \multicolumn{1}{c}{Val.\ [$P(Y,\cG^\text{tr}_{N^\text{tr}})$]} & \multicolumn{1}{c}{\textbf{Test} ($\uparrow$) [$P(Y,\cG^\text{te}_{N^\text{te}})$]}
		 \\
		\cmidrule(lr){2-4} \cmidrule(lr){5-7} \cmidrule(lr){8-10}
		PNA & 1.00 (0.00) & 1.00 (0.00) & 0.65 (0.10)
		    & 0.96 (0.06) & 0.94 (0.03) & 0.57 (0.19)
		    & 0.99 (0.01) & 1.00 (0.00) & 0.69 (0.19) \\
		PNA (mean xu-readout) & 1.00 (0.00) & 1.00 (0.00) & 0.86 (0.13)
		                & 0.97 (0.02) & 0.95 (0.02) & 0.64 (0.11)
		                & 0.99 (0.01) & 1.00 (0.00) & 0.70 (0.15) \\
		PNA (max xu-readout) & 0.99 (0.01) & 0.97 (0.02) & 0.83 (0.13)
		                & 0.94 (0.04) & 0.93 (0.03) & 0.80 (0.12)
		                & 0.95 (0.05) & 0.95 (0.05)  & 0.80 (0.15) \\
		PNA + IRM & NA & NA & NA
		             & 0.95 (0.05) & 0.94 (0.03) & 0.58 (0.19)
		             & 0.99 (0.01) & 1.00 (0.00) & 0.70 (0.20) \\
		GCN & 0.99 (0.01) & 0.98 (0.02) & 0.62 (0.09)
		    & 0.95 (0.02) & 0.96 (0.02) & 0.55 (0.17)
		    & 1.00 (0.00) & 1.00 (0.00) & 0.73 (0.17) \\
		GCN (mean xu-readout) & 0.94 (0.03) & 0.99 (0.01) & 0.61 (0.12)
		                & 0.93 (0.05) & 0.94 (0.02) & 0.69 (0.20)
		                & 1.00 (0.00) & 1.00 (0.00) & 0.84 (0.13) \\
		GCN (max xu-readout) & 0.99 (0.01) & 1.00 (0.00) & 0.76 (0.07)
		                & 0.95 (0.04) & 0.98 (0.02) & 0.61 (0.17)
		                & 0.98 (0.02) & 1.00 (0.00) & 0.70 (0.20) \\
		GCN + IRM & NA & NA & NA
		            & 0.93 (0.05) & 0.97 (0.03) & 0.65 (0.19)
		            & 1.00 (0.00) & 1.00 (0.00) & 0.84 (0.17) \\
		GIN & 0.97 (0.02) & 1.00 (0.00) & 0.64 (0.17)
		    & 0.95 (0.03) & 0.96 (0.04) & 0.66 (0.20)
		    & 0.98 (0.02) & 1.00 (0.00) & 0.74 (0.19) \\
		GIN (mean xu-readout) & 1.00 (0.00) & 1.00 (0.00) & 0.85 (0.14)
		                & 0.97 (0.01) & 0.99 (0.01) & 0.75 (0.18)
		                & 0.99 (0.01) & 1.00 (0.00) & 0.80 (0.15) \\
		GIN (max xu-readout) & 0.95 (0.02) & 0.97 (0.03) & 0.67 (0.18) 
		                & 0.93 (0.06) & 0.94 (0.03) & 0.67 (0.17)
		                & 0.99 (0.01) & 1.00 (0.00) & 0.69 (0.15) \\
		GIN + IRM & NA & NA & NA
		            & 0.95 (0.03) & 0.97 (0.04) & 0.64 (0.19)
		            & 0.98 (0.02) & 1.00 (0.00) & 0.75 (0.19) \\
		RPGIN & 0.98 (0.02) & 1.00 (0.00) & 0.49 (0.15)
		        & 0.96 (0.03) & 0.99 (0.01) & 0.54 (0.12)
		        & 0.99 (0.01) & 1.00 (0.00) & 0.50 (0.13) \\
		\hline
		WL Kernel & 1.00 (0.00) & 0.95 (0.00) & 0.57 (0.00)
		            & 0.99 (0.00) & 0.90 (0.00) & 0.62 (0.00)
		            & 1.00 (0.00) & 1.00 (0.00) & 0.57 (0.00) \\
		GC Kernel & 1.00 (0.00) & 0.90 (0.00) & 0.43 (0.00)
		            & 1.00 (0.00) & 0.80 (0.00) & 0.43 (0.00)
		            & 0.99 (0.00) & 0.90 (0.00) & 0.43 (0.00) \\
		\hline
		$\Gamma_\text{1-hot}$ & 1.00 (0.00) & 0.90 (0.00) & 0.50 (0.07) 
		                        & 0.97 (0.03) & 0.85 (0.05) & 0.50 (0.07)
		                        & 0.98 (0.00) & 0.96 (0.02) & 0.45 (0.05) \\
		$\Gamma_\text{GIN}$ & 1.00 (0.00)  & 1.00 (0.00) & \textbf{0.98 (0.02)}
		                    & 0.96 (0.02)  & 0.95 (0.01) & \textbf{0.95 (0.06)}
		                    & 1.00 (0.00) & 1.00 (0.00) & 0.88 (0.12)  \\
		$\Gamma_\text{RPGIN}$ & 1.00 (0.00) & 1.00 (0.00) & \textbf{1.00 (0.00)}
		                        & 0.97 (0.03) & 0.95 (0.02) & \textbf{0.95 (0.05)}
		                        & 1.00 (0.00) & 1.00 (0.00) & \textbf{0.93 (0.05)} \\
		\bottomrule
		\end{tabular}
		}
	\end{center}
	\end{sc}
	\end{small}
\vspace{-10pt}
\end{table*}

\section{Empirical Results}
\label{sec:experiments}
This section is dedicated to the empirical evaluation of our theoretical claims, including the ability of the representations
in \Cref{eq:Gamma-1hot,eq:Gamma-kgnn,eq:Gamma-krpgnn} to extrapolate as predicted by \Cref{prop:Einv} for tasks
that abide by \Cref{def:trainG,def:testG}. Due to space constraints, our results are summarised here, while further details are relegated to \Cref{sec:ExperimentAppendix}.
\update{Our code is also available\footnote{\small \url{https://github.com/PurdueMINDS/size-invariant-GNNs}}}.

We explore the extrapolation power of $\Gamma_\text{1-hot}$, $\Gamma_\text{GIN}$ and $\Gamma_\text{RPGIN}$ of \Cref{eq:Gamma-1hot,eq:Gamma-kgnn,eq:Gamma-krpgnn} using the Graph Isomorphism Network (GIN)~\citep{xu2018powerful} as our base GNN model, and Relational Pooling GIN (RPGIN)~\citep{pmlr-v97-murphy19a} as a more expressive GNN. The graph representations are then passed to a $L$-hidden layer feedforward neural network (MLP) with softmax outputs that give the predicted classes, $L\in\{0, 1\}$.
\update{As described in~\Cref{sec:intuitive}, we obtain induced homomorphism densities of {\em connected} graphs. For practical reasons, we focus only on densities of graphs of size {\em exactly} $k$, which is treated as a hyperparameter. Note that the number of parameters for our $\Gamma_\text{GNN}$ and $\Gamma_{\text{GNN}^+}$ does not depend on $k$ (for $\Gamma_\text{1-hot}$ it does), and the forward pass on
the $k$-sized graphs can be performed in parallel.}

\vspace{-5pt}
\paragraph{Baselines.} Our baselines include the Graphlet Counting kernel~(GC Kernel)~\citep{shervashidze2009efficient}, which uses the $\Gamma_\text{1-hot}$ representation as input to a downstream classifier.
We report $\Gamma_\text{1-hot}$ separately from {GC Kernel} since $\Gamma_\text{1-hot}$ differs from {GC Kernel} in that we add the same feedforward neural network (MLP) classifier used in the $\Gamma_\text{GNN}$ model.
We also include GIN~\citep{xu2018powerful},
GCN~\citep{Kipf2016} and PNA~\citep{corso2020principal}, considering the sum, mean, and max READOUTs as proposed by~\citet{xu2020neural} for extrapolations (which we denote as {\em XU-READOUT} to not confuse with our $\text{READOUT}_\Gamma$). We also examine a more-expressive GNN, RPGIN~\citep{pmlr-v97-murphy19a}, and the WL Kernel~\citep{shervashidze2011weisfeiler}.
We do not use the method of~\citet{yehudai2020size} as a baseline since it is a covariate shift adaptation approach that requires samples from ${\rm P}(\cG^\text{te}_{N^\text{te}})$, which are not available in our setting.

\vspace{-5pt}
\paragraph{Experiments with single and multiple graph sizes in training.}
Our single-environment experiments consist of a single graph size in training, and different sizes in test (different from the training size).
Whenever multiple environments are available in training ---multiple environments implies different graph sizes---, we employ Invariant Risk Minimization (IRM), considering the penalty proposed by~\citet{arjovsky2019invariant} for each environment (defined empirically as a range of training examples with similar graph sizes).

For each task, we report
\begin{enumerate*}[label=(\alph*)]
\item \emph{training} accuracy
\item \emph{validation} accuracy, which are new examples sampled from ${\rm P}(Y,\cG^\text{tr}_{N^\text{tr}})$; and
\item \emph{extrapolation test} accuracy, which are new OOD examples sampled from ${\rm P}(Y,\cG^\text{te}_{N^\text{te}})$.
\end{enumerate*}
\update{In our experiments we perform early stopping as per~\citet{hu2020open}}. 

\subsection{Size extrapolation tasks for unattributed graphs}
\label{sec:unattributed}

{\em Schizophrenia task.} %
We use the fMRI brain graph data on 71 schizophrenic patients and 74 controls for classifying individuals with schizophrenia~\citep{de2016mapping}.
Vertices represent brain regions (voxels) with edges as functional connectivity. 
We process the graph differently between training and test data, where training graphs have exactly 264 vertices (a single environment) and \update{control-group graphs in test} have around 40\% fewer vertices. We employ a 5-fold cross-validation for hyperparameter tuning.

{\em \erdosrenyi task.} %
We simulate \erdosrenyi graphs ~\citep{gilbert1959random, erdds1959random}  
as a simple graphon random graph model.
The task is to classify the edge probability $p \in \{0.2, 0.5, 0.8 \}$ of the generated graph.
First we consider a single-environment version of the task, where we train and validate on graphs of size 80 and extrapolate to graphs with size 140 in test.
We also consider another experiment with training/validation graph sizes uniformly selected from $\{70,80\}$ (so we can use IRM), with the test data same as before (graphs of size 140 in test). 

{\bf Results.} 
\Cref{tab:unattributed} shows that all methods perform well in validation (generalization over the training distribution). However, only $\Gamma_\text{1-hot}$ (GC Kernel and our simple classifier), $\Gamma_\text{GIN}$, $\Gamma_\text{RPGIN}$ are able to extrapolate, while displaying very similar ---often identical--- accuracies in validation (sampled from ${\rm P}(\cG^\text{tr}_{N^\text{tr}})$) and test (sampled from ${\rm P}(\cG^\text{te}_{N^\text{te}})$) in all experiments, as predicted by combining the theoretical results in \Cref{prop:Einv} and \Cref{thm:sizeExtrapolationBound}. 
Using IRM in the \erdosrenyi task shows no improvement over not using IRM in the multi-environment setting.

\begin{table}[t]
	\caption{Extrapolation performance over real-world graph datasets with OOD tasks violating \Cref{def:trainG,def:testG} and conditions of \Cref{thm:sizeExtrapolationBound}. \update{Always one of our E-invariant representations $\Gamma_\text{GIN}$ and $\Gamma_\text{RPGIN}$ is amongst the top 4 best methods in all datasets except \textsc{NCI109}}. Table shows mean (standard deviation) Matthews correlation coefficient (MCC) of the classifiers over the OOD test data. Bold emphasises the top-4 models (in average MCC) for each dataset.}
    \vspace{-5pt}
	\label{REAL-WORLD-MATTH}
    \begin{small}
	\begin{sc}
	\begin{center}
	\resizebox{\columnwidth}{!}{
    \begin{tabular}{lrrrr}
    \toprule
    \textbf{Datasets} &  \multicolumn{1}{c}{\textbf{NCI1}} & \multicolumn{1}{c}{\textbf{NCI109}} & \multicolumn{1}{c}{\textbf{PROTEINS}} & \multicolumn{1}{c}{\textbf{DD}} \\
    \midrule
    RANDOM                  & 0.00 (0.00) & 0.00 (0.00) & 0.00 (0.00) & 0.00 (0.00) \\
    PNA                   &   0.21 (0.06) &   \textbf{0.24 (0.06)} &   \textbf{0.26 (0.08)} &   0.24 (0.10) \\
    PNA (mean xu-readout) &   0.12 (0.05) &   0.21 (0.04) &   0.25 (0.06) &   \textbf{0.29 (0.08)} \\
    PNA (max xu-readout)  &   0.16 (0.05) &   0.18 (0.07) &   0.20 (0.05) &   0.12 (0.14) \\
    PNA + IRM             &   0.21 (0.07) &   \textbf{0.27 (0.08)} &   \textbf{0.26 (0.10)} &   \textbf{0.26 (0.08)} \\
    GCN                   &   0.20 (0.06) &   0.15 (0.06) &   0.21 (0.09) &   0.23 (0.05) \\
    GCN (mean xu-readout) &   0.20 (0.04) &   0.15 (0.09) &   0.23 (0.07) &   0.19 (0.06) \\
    GCN (max xu-readout)  &   0.20 (0.04) &   0.19 (0.07) &   0.20 (0.14) &   0.09 (0.08) \\
    GCN + IRM             &   0.12 (0.05) &   \textbf{0.22 (0.06)} &   0.20 (0.07) &   0.23 (0.07) \\
    GIN                   &   \textbf{0.25 (0.06)} &   0.18 (0.05) &   0.23 (0.05) &   0.25 (0.09) \\
    GIN (mean xu-readout) &   0.16 (0.05) &   0.14 (0.05) &   0.24 (0.05) &   \textbf{0.27 (0.12)} \\
    GIN (max xu-readout)  &   0.15 (0.08) &   0.18 (0.08) &  \textbf{0.28 (0.11)} &   0.19 (0.07) \\
    GIN + IRM             &   0.18 (0.08) &   0.16 (0.04) &   \textbf{0.26 (0.06)} &   0.21 (0.09) \\
    RPGIN                 &   0.15 (0.04) &   0.19 (0.05) &   0.24 (0.09) &   0.22 (0.09) \\
    \hline
    WL kernel               & \textbf{0.39 (0.00)} & 0.21 (0.00) & 0.00 (0.00) & 0.00 (0.00) \\
    GC kernel               & 0.02 (0.00) & 0.01 (0.00) & \textbf{0.29 (0.00)} & 0.00 (0.00) \\
    \hline
    $\Gamma_\text{1-hot}$      &   0.17 (0.08) &   \textbf{0.25 (0.06)} &   0.12 (0.09) &   0.23 (0.08) \\
    $\Gamma_\text{GIN}$               &   \textbf{0.24 (0.04)} &   0.18 (0.04) &   \textbf{0.29 (0.11)} &   \textbf{0.28 (0.06)} \\
    $\Gamma_\text{RPGIN}$             &   \textbf{0.26 (0.05)} &   0.20 (0.04) &   0.25 (0.12) &   0.20 (0.05) \\
\bottomrule
    \end{tabular}
}
\end{center}
\end{sc}
\end{small}
\vspace{-8pt}
\end{table}
\subsection{Size/attribute extrapolation for attributed graphs}
\label{sec:attribute}
We now define a Stochastic Block Model (SBM) task with vertex attributes. The SBM has two blocks. Our goal is to classify the cross-block edge probability $\mP_{1,2} = \mP_{2,1} \in \{0.1, 0.3\}$ of a sampled graph.
Vertex attribute distributions depend on the blocks.
In block 1 vertices are randomly assigned red and blue attributes, while in block 2 vertices are randomly assigned green and yellow attributes (see {\bf SBM with vertex attributes} in \Cref{sec:family}).

The change in environments between training and test introduces a joint attribute-and-size distribution shift: In training, the vertices are $90\%$ red (resp.\ green) and $10\%$ blue (resp.\ yellow) in block 1 (resp.\ block 2). While in test, the distribution is flipped and vertices are $10\%$ red (resp.\ green) and $90\%$ blue (resp.\ yellow) in block 1 (resp.\ block 2).
We consider three scenarios, with the same test data made of graphs of size 40:
\begin{enumerate*}
    \item[(a)] A single-environment case, where all training graphs have size 20;
    \item[(b)] A multi-environment case, where training graphs have sizes 14 and 20;
    \item[(c)] A multi-environment case, where training graphs have sizes 20 and 30. These differences in training data will check whether having graphs of sizes closer to the test graph sizes improves the performance of traditional graph representation methods.
\end{enumerate*}

{\bf Results.} \Cref{tab:attr} shows how traditional graph representations and $\Gamma_\text{1-hot}$ (both GC Kernel and our neural classifier) tap into the easy correlation between $Y$ and the density of red and green vertex attributes in the training graphs, while $\Gamma_\text{GIN}$ and $\Gamma_\text{RPGIN}$, with their attribute regularization (\Cref{eq:regul}), are approximately E-invariant, resulting in higher test accuracy that more closely matches their validation accuracy. Moreover, applying IRM has no beneficial impact, while adding larger graphs in training (closer to test graph sizes) increases the extrapolation  accuracy of most methods. 

\subsection{Experiments with real-world datasets that violate our causal model}
Finally, we test our E-invariant representations on datasets that violate \Cref{def:trainG,def:testG} and the conditions of \Cref{thm:sizeExtrapolationBound}. 
We consider four vertex-attributed datasets (\textsc{NCI1}, \textsc{NCI109}, \textsc{DD}, \textsc{PROTEINS}) from~\citet{Morris2020}, and split the data as proposed by~\citet{yehudai2020size}. 
As mentioned earlier,~\citet{yehudai2020size} is not part of our baselines since it requires samples from the test distribution ${\rm P}(\cG^\text{te}_{N^\text{te}})$.

Training and test data are created as follows: Graphs with sizes smaller than the $50$-th percentile are assigned to training, while graphs with sizes larger than the $90$-th percentile are assigned to test. A validation set for hyperparameter tuning consists of $10\%$ held out examples from training.

{\bf Results.} \Cref{REAL-WORLD-MATTH} shows the test results using the Matthews correlation coefficient (MCC) --- MCC was chosen due to significant class imbalances in the OOD shift of our test data, see \Cref{sec:ExperimentAppendix} for more details. \update{We observe that always one of our E-invariant representations $\Gamma_\text{GIN}$ and $\Gamma_\text{RPGIN}$ is amongst the top 4 best methods in all datasets
except \textsc{NCI109}. We also note that the \textsc{WL Kernel} performs really well at \textsc{NCI1} and very poorly (random) on \textsc{PROTEINS} and  \textsc{DD}, showcasing the importance of consistency across datasets}.

\update{{\em Comments on \Cref{REAL-WORLD-MATTH}.} Counterfactual-driven extrapolations have their representation methods tailored to a specific extrapolation mechanism. Unlike in-distribution tasks (and covariate shift adaptation tasks, where one sees test distribution examples of the input graphs), counterfactual-driven extrapolations rely on being robust to the distribution-shift mechanism given by the causal model. Hence, it is expected that the causal extrapolation mechanism that works for a molecular task may not work as well for a social network (unless they share a universal graph-formation mechanism). The schizophrenia task~(\Cref{sec:unattributed}) has the same mechanism as our causal model (hence, good performance). Further research may show that every single dataset in this subsection has its own distinct extrapolation mechanism. We think that although these datasets violate our assumptions, this subsection is important (and we hope will be copied by future work) to show which datasets may need different extrapolation mechanisms.}

\vspace{-5pt}
\section{Conclusions}
\label{sec:conclusions}
\vspace{-2pt}
In this work we looked at the task of out-of-distribution (OOD) graph classification, where train and test data have different distributions.
By introducing a structural causal model inspired by graphon models~\citep{lovasz2006limits}, we defined a representation that is approximately invariant to the train/test distribution changes of our causal model, empirically showing its benefits on both synthetic and real-world datasets against standard graph classification baselines.
Finally, our work contributed a blueprint for defining graph extrapolation tasks through causal models.

\section*{Acknowledgements}
This work was funded in part by the National Science Foundation (NSF) awards CAREER IIS-1943364 and CCF-1918483, the Frederick N.\ Andrews Fellowship, and the Wabash Heartland Innovation Network.  Any opinions, findings and conclusions or recommendations expressed in this material are those of the authors and do not necessarily reflect the views of the sponsors.
We would like to thank our reviewers, who gave excellent suggestions to improve the paper.
Further, we would like to thank Ryan Murphy for many insightful discussions, and Mayank Kakodkar and Carlos H. C.\ Teixeira for their invaluable help with the subgraph function estimation.


\bibliography{motif_rp,invariances,icmlbib,causality} 
\bibliographystyle{icml2021}

\newpage

\appendix
\twocolumn[
\icmltitle{Supplementary Material}]
\section{Proof of \Cref{prop:Einv}}
\propEinv*
\begin{proof}
First note that $Y$ is only a function of $W$ and an independent random noise (following \Cref{def:trainG,def:testG}, depicted in \Cref{fig:SCM}).
\update{Therefore, $Y$ is E-invariant, and thus
\begin{equation*}
\begin{split}
{\rm P}(Y|\cG^\text{te}_{N^\text{te}} = G_{N^\text{te}}^\text{te})
={\rm P}(Y|\cG^\text{tr}_{N^\text{tr}}=G_{N^\text{tr}}^\text{tr}),
\end{split}
\end{equation*}
since the observed graphs in test ($G_{N^\text{te}}^\text{te}$) and training ($G_{N^\text{tr}}^\text{tr}$) only differ due to the change of environments while sharing the same graphon variable $W$ and the other random variables.

The definition of E-invariance states that $\forall e \in \supp(E^\text{tr})$, $\forall e^\dagger \in \supp(E^\text{te})$,
\[
\Gamma(\cG^\text{tr}_{N^\text{tr}}|E^\text{tr}=e)=\Gamma(\cG^\text{te}_{N^\text{te}}|E^\text{te}=e^\dagger).
\]
So, the E-invariance of $\Gamma$ yields
$$\rho(y, \Gamma(\cG^\text{tr}_{N^\text{tr}})) = \rho(y, \Gamma(\cG^\text{te}_{N^\text{te}})),$$ concluding our proof.}
\end{proof}

\section{Proof of \Cref{thm:sizeExtrapolationBound}}
\thmsizeExtrapolationBound*
\begin{proof}
\newcommand{\graphG}{\cG^\wcard_{N^\wcard}}
\newcommand{\graphGa}{\cG^\wcard_{N^\wcard}}
\newcommand{\graphGm}{\cG^\wcard_{m}}
\newcommand{\graphGtr}{\cG^\text{tr}_{N^\text{tr}}}
\newcommand{\graphGte}{\cG^\text{te}_{N^\text{te}}}
\newcommand{\tinj}{t_{\text{ind}}}
We first replace $t_{\rm ind}$ by $t_{\rm inj}$, which is
defined by
\begin{equation}
    t_{\rm inj}(F_k,\cG^\wcard_{N^\wcard})=\frac{{\rm inj}(F_k,\cG^\wcard_{N^\wcard})}{N^\wcard!/(N^\wcard - k)!},
\end{equation}
where $\text{inj}(F_k,\cG^\wcard_{N^\wcard})$ is the number of injective homomorphisms of $F_k$ into $\cG^\wcard_{N^\wcard}$. 
Then, we know from~\citet[Theorem 2.5]{lovasz2006limits} that for unattributed graphs $\cG^\wcard_{N^\wcard}$,
\begin{equation}
    {\rm P}(|t_{\text{inj}}(F_k,\graphG)-t(F_k,W)|>\epsilon)\leq 2\exp(-\frac{\epsilon^2}{2k^{2}}N^{\wcard}),
    \label{eq:bound}
\end{equation}
\update{where $W$ is the graphon function as illustrated in \Cref{def:trainG,def:testG}. As defined in~\citet{lovasz2006limits},
$$t(F_k,W) = \int_{[0,1]^{k}} \prod_{ij\in E(F_k)}W(x_i,x_j)dx_1\cdots dx_{k},$$ 
where $E(F_k)$ denotes the edge set of $F_k$.} This bound shows that $t_{\text{inj}}(F_k,\graphG)$ converges to $t(F_k,W)$ as $N^\wcard\rightarrow \infty$. Actually, we can get similar bounds for $t_{\rm ind}$ using a similar proof technique. \update{Although the value it converges to is different, the difference between values is preserved as it will be proved in the following text}. More importantly, it can be extended to vertex-attributed graphs under our SCM assumptions depicted in \Cref{def:trainG,def:testG}.

We can have this extension because, for the vertex-attributed graphs in \Cref{def:trainG,def:testG}, $g_X$ operates on attributed graphs similarly as the graphon does on unattributed graphs. We can consider the graph generation procedure as first generating the underlying structure, and then adding vertex attribute accordingly to its corresponding random graphon value $U_v \in \text{Uniform}(0,1)$ and graphon $W$. $g_X(\cdot,\cdot)$ being invariant to $E^\text{tr}$ and $E^\text{te}$ means for any two environments $e\in \supp(E^\text{tr}), e^\dagger \in \supp(E^\text{te})$, $g_X(e,\cdot)=g_X(e^\dagger,\cdot)$. 

Then for a given vertex-attributed graph $F_k$ with $k$ vertices, and a given (whole) graph size $N^\wcard$, we can define $\phi$ as an induced map $\phi: [k]\rightarrow [N^\wcard]$, \update{which can be thought about as how the $k$ vertices in $F_k$ are mapped to the vertices in $\graphG$.} 
Define $C_\phi = 1$ if $\phi$ is a homomorphism from $F_k$ to the $W$-random graph $\graphG$, otherwise $C_\phi = 0$. 
We define $\graphGm$ as the subgraph of $\graphGa$ induced by vertices $\{1,...,m\}$. \update{Note here $m$ has two meanings. First, it represents the $m$-th vertex. Second, it indicates the size of the subgraph.} We define $B_m=\frac{1}{\binom{N^\wcard}{k}}\sum_\phi {\rm P}(C_\phi=1|\graphGm), 0\leq m\leq N^\wcard$ as the expected induced homomorphism densities once we observe the subgraph $\graphGm$. Here $B_0=\frac{1}{\binom{N^\wcard}{k}} \sum_\phi {\rm P}(C_\phi=1)$ denotes the expectation before we observe any vertices.

$B_m$ is a martingale for unattributed graphs~\citep[Theorem 2.5]{lovasz2006limits}. And since in \Cref{def:trainG,def:testG} we also use the graphon $W$ and $g_X$ operates on attributed graphs using the graphon $W$ and $U_v$, it is also a martingale for vertex-attributed graphs based on our definitions. We do not need to care about the environment variable $E^\wcard$ here because the function $g_X$ is invariant to $E^\wcard$ and, therefore, it can be treated as a constant. Then,
\begin{align*}
    |B_m&\!-B_{m-1}|\!\\ =&\frac{1}{\binom{N^\wcard}{k}}|\sum_\phi {\rm P}(C_\phi=1|\graphGm)-{\rm P}(C_\phi=1|\cG^\wcard_{m-1})|
    \\\leq &\frac{1}{\binom{N^\wcard}{k}}\sum_\phi| {\rm P}(C_\phi=1|\graphGm)-{\rm P}(C_\phi=1|\cG^\wcard_{m-1})|.
\end{align*}
Here, for each $\phi: [k]\rightarrow [N^\wcard]$ that does not contain \update{the value $m$ in its image (which means no vertex in $F_k$ is mapped to the $m$-th vertex in $\graphG$)},
the difference is $0$. For all other terms, the terms are at most $1$. Thus,
$$|B_m-B_{m-1}|\leq \frac{\binom{N^\wcard-1}{ k-1}}{\binom{N^\wcard}{k}}=\frac{k}{n}.$$

By definition, $B_0 = \frac{1}{\binom{N^\wcard}{k}}\sum_\phi{\rm P}(C_\phi=1) = t^\wcard(F_k,W)$, and $B_{N^\wcard}=\frac{1}{\binom{N^\wcard}{k}}\text{ind}(F_k,\graphG)=t_{\text{ind}}(F_k,\graphG)$, where $t^\wcard(F_k,W)$ is defined as $B_0$, is the expected induced homomorphism densities if we only know the graphon $W$ and we did not observe any vertex in the graph. 

Then, we can use Azuma's inequality for Martingales,
\begin{equation*}
\begin{split}
    {\rm P}(B_{N^\wcard}-B_0>\epsilon) &\leq \exp(-\frac{\epsilon^2}{2N^\wcard(k/N^\wcard)^2}) \\ &= \exp(-\frac{\epsilon^2}{2k^2}N^\wcard).
\end{split}
\end{equation*}
Since $B_{N^\wcard} = t_{\text{ind}}(F_k,\graphG)$, and $B_0 = t^\wcard(F_k,W)$, we get the similar bound as in \Cref{eq:bound},
$${\rm P}(|t_{\text{ind}}(F_k,\graphG)-t^\wcard(F_k,W)|>\epsilon)\leq 2\exp(-\frac{\epsilon^2}{2k^{2}}N^{\wcard}).$$
Since $|\tinj(F_k,\graphGtr)-t^\wcard(F_k,W)|\leq \frac{\epsilon}{2}$,  $|\tinj(F_k,\graphGte)-t^\wcard(F_k,W)|\leq \frac{\epsilon}{2}$ imply $|\tinj(F_k,\graphGtr)-\tinj(F_k,\graphGte)|\leq \epsilon$, we have,
\begin{equation}
\label{eq:bound1}
\begin{split}
    &{\rm P}(|\tinj(F_k,\graphGtr)-\tinj(F_k,\graphGte)|>\epsilon)\\&  = 1 - {\rm P}(|\tinj(F_k,\graphGtr)-\tinj(F_k,\graphGte)|\leq\epsilon) \\
    &\leq 1 - {\rm P}(|\tinj(F_k,\graphGtr)-t^\wcard(F_k,W)|\leq \frac{\epsilon}{2})\\
    &\qquad \cdot{\rm P}(|\tinj(F_k,\graphGte)-t^\wcard(F_k,W)|\leq \frac{\epsilon}{2})\\
    & \leq 1-(1-2\exp(-\frac{\epsilon^2}{8k^2}N^\text{tr}))(1-2\exp(-\frac{\epsilon^2}{8k^2}N^\text{te}))\\&=2(\exp(-\frac{\epsilon^2}{8k^2}N^\text{tr})+\exp(-\frac{\epsilon^2}{8k^2}N^\text{te}))\\
    &\qquad -4\exp(-\frac{\epsilon^2}{8k^2}(N^\text{tr}+N^\text{te}))\\& \leq 2(\exp(-\frac{\epsilon^2}{8k^2}N^\text{tr})+\exp(-\frac{\epsilon^2}{8k^2}N^\text{te})).
    \end{split}
    \raisetag{1.1\baselineskip}
\end{equation}
Then we know,
\begin{equation}
\label{eq:bound2}
    \begin{split}
        &{\rm P}(||\Gamma_\text{1-hot}(\graphGtr)-\Gamma_\text{1-hot}(\graphGte)||_\infty\leq\epsilon)\\&={\rm P}(|\tinj(F_{k'},\graphGtr)-\tinj(F_{k'},\graphGte)|\leq \epsilon, \forall  F_{k'}\in \allpossibleGraphsK)\\&\geq 1 - \sum_{F_{k'}\in \allpossibleGraphsK} {\rm P}(|\tinj(F_{k'},\graphGtr)-\tinj(F_{k'},\graphGte)|>\epsilon)\\&\geq 1-2|\allpossibleGraphsK|(\exp(-\frac{\epsilon^2N^\text{tr}}{8k^2})+\exp(-\frac{\epsilon^2N^\text{te}}{8k^2})).
    \end{split}
    \raisetag{1.1\baselineskip}
\end{equation}
It follows from the Bonferroni inequality that ${\rm P}(\cap_{i=1}^N A_i)\geq 1-\sum_{i=1}^{N}{\rm P(\tilde{A}_i)}$, where $A_i$ and its complement $\tilde{A}_i$ are any events.
Therefore, 
\begin{align*}
    {\rm P}(||\Gamma_\text{1-hot}&(\graphGtr)-\Gamma_\text{1-hot}(\graphGte)||_\infty>\epsilon)\\&\leq 2|\allpossibleGraphsK|(\exp(-\frac{\epsilon^2N^\text{tr}}{8k^2})+\exp(-\frac{\epsilon^2N^\text{te}}{8k^2})),
\end{align*}
concluding the proof.

\end{proof}

\section{Biases in estimating induced homomorphism densities}
\label{sec:BiasAppendix}
Induced (connected) homomorphism densities \update{of a given graph $F_{k'}$} over all possible $k'$-vertex ($k'\leq k$) connected graphs for an $N^\wcard$-vertex graph $\cG_{N^\wcard}^\wcard$ are defined as
\[
    \omega(F_{k'}, \cG_{N^\wcard}^\wcard)=\frac{{\rm ind}(F_{k'},\cG_{N^\wcard}^\wcard)}{\sum_{F_{k'}\in \mathcal{F}_{\leq k}} {\rm ind}(F_{k'},\cG_{N^\wcard}^\wcard)}.
\]
This is a slightly different definition from the induced homomorphism densities in~\Cref{eq:tinj}. In the main text, the denominator is the total number of possible mappings (which can include mappings that are disconnected). Here we consider the total numbers of induced mappings that are connected as is common practice.

Achieving unbiased estimates for induced (connected) homomorphism densities usually requires sophisticated methods and \update{significant} amount of time. We show that a biased estimator can also work for the $\text{GNN}^+$ in \Cref{eq:Gamma-krpgnn} if the bias is multiplicative and the $\text{READOUT}_\Gamma$ is simply the sum of the vertex embeddings. We formalize it as follows.

\begin{restatable}{proposition}{propbias}
\label{prop:bias} 
Assume $\hat{\omega}(F_{k'},\cG_{N^\wcard}^\wcard)$ is a biased estimator for $\omega(F_{k'},\cG_{N^\wcard}^\wcard)$ for any $k'$ and $k'$-sized connected graphs $F_{k'}$ in a $N^\wcard$-vertex $\cG_{N^\wcard}^\wcard$, such that $\mathbb{E}(\hat{\omega}(F_{k'},\cG_{N^\wcard}^\wcard))=\beta(F_{k'})\omega(F_{k'}, \cG_{N^\wcard}^\wcard)$, where $\beta(F_{k'})$ ($\beta(\cdot)>0$) is the bias related to the graph $F_{k'}$, and the expectation is over the sampling procedure.
The expected learned representation $\mathbb{E}(\sum_{F_{k'}\in \mathcal{F}_{\leq k}} \hat{\omega}(F_{k'},\cG_{N^\wcard}^\wcard ) \mathbf{1}^{\rm T}({\text{GNN}^+}(F_{k'})))$ 
can be the same as using the true induced (connected) homomorphism densities $\omega(F_{k'},\cG_{N^\wcard}^\wcard),\forall F_{k'}\in \mathcal{F}_{\leq k}$.
\end{restatable}
\begin{proof}
\update{W.L.O.G, assume ${\text{GNN}^+_{0}}(F_{k'})$ is the representation we can learn from the true induced (connected) homomorphism densities $\omega(F_{k'},\cG_{N^\wcard}^\wcard),\forall F_{k'}\in \mathcal{F}_{\leq k}$.
When only using the biased estimators, if we are able to learn the representation ${\text{GNN}^+}(F_{k'})={ \text{GNN}^+_{0}}(F_{k'})/\beta(F_{k'})$ for all $F_{k'}\in \mathcal{F}_{\leq k}$, then we can still get the graph representation in \Cref{eq:Gamma-krpgnn} the same as using the true induced (connected) homomorphism densities. }
This is possible because $\text{GNN}^+$ is proven to be a most expressive $k'$-vertex graph representation, thus it is able to learn any function on the graph $F_{k'}$. Then,
\begin{equation}
\begin{aligned}[b]
    \E\left[\sum_{F_{k'}\in \mathcal{F}_{\leq k}} \hat{\omega}(F_{k'},\cG_{N^\wcard}^\wcard) \mathbf{1}^{\rm T}({\text{GNN}^+}(F_{k'}))\right] =
    \\\sum_{F_{k'}\in \mathcal{F}_{\leq k}} \omega(F_{k'},\cG_{N^\wcard}^\wcard) \mathbf{1}^{\rm T}({\text{GNN}^+_{0}}(F_{k'})),
\end{aligned}
\end{equation}
where $\mathbf{1}^{\rm T}({\text{GNN}^+}(F_{k'}))$ is the sum of the vertex embeddings given by the $\text{GNN}^+$ if it is an equivariant representation of the graph.
\end{proof}

\section{Review of Graph Neural Networks}
\label{sec:GnnTutorial}
Graph Neural Networks (GNNs) constitute a popular class of methods for learning representations of vertices in a graph or graph-wide representations~\citep{Kipf2016,Atwood2016, Hamilton2017, gilmer17a,velickovic2018graph, xu2018powerful, morris2019weisfeiler, pmlr-v97-you19b, liu2019hyperbolic, chami2019hyperbolic}.  
Graph-wide representations can also be obtained by applying GNNs to the connected induced subgraphs in a larger graph and then averaging the resulting subgraph representations.  
That is, in our work, we have applied GNNs to connected induced subgraphs in a graph, and then aggregated (averaged) them to obtain the representation of the graph.  We briefly summarize the idea, but more details can be found in texts such as by~\citet{hamilton2020book} and reviews by~\citet{wu2020comprehensive} and~\citet{zhang2020deeplearningongraphs} and the references therein.

Suppose we have a graph $G$ with vertex set $V = \{1,  \dots, N\}$, and each vertex in our data may carry some vertex attribute (also called a \emph{feature}).  For instance, in a molecule, vertices may represent atoms, edges may represent bonds, and features may indicate the atomic number~\citep{Duvenaud2015}. These vertex features can be stored in an $N \times d$ matrix $\mX$, where $d$ is the dimension of the vertex feature vector.  In particular, row $v \in V$ of $X_v$ holds the attribute associated with vertex $v$.

Roughly speaking, GNNs proceed by passing messages among vertices, later passing the result through a learnable function such as an MLP, and repeating $T \in \mathbb{Z}_{\ge 1}$ times.  At each iteration $t = \{1, 2, \ldots, T\}$, all vertices $v \in V$ are associated with a learned vector $\vh^{(t)}$. Specifically, we begin by initializing a vector as $\vh_{v}^{(0)} = X_{v}$ for every vertex $v \in V$. Then, we recursively compute an update such as the following
\begin{equation}\label{eq:app_gnn}
    \vh^{(t)}_v =\mathrm{MLP}^{(t)} \Big(\vh^{(t-1)}_v, \sum_{ _{u \in \mathcal{N}(v)} } \vh^{(t-1)}_u \Big) , \quad \forall v \in V,
\end{equation}
where $\mathcal{N}(v) \subseteq V$ denotes the neighborhood set of $v$ in the graph, $\mathrm{MLP}^{(t)}$ denotes a multi-layer perceptron, and whose superscript $t$ indicates that the MLP at each recursion layer may have different learnable parameters.  We can replace the summation with any permutation-invariant function of the neighborhood.  We see that GNNs recursively update vertex states with states from their neighbors and their state from the previous recursion layer.  Additionally, we can sample from the neighborhood set rather than aggregating over every neighbor.  Generally speaking there is much research into the variations of this recursion step and we refer the reader to aforementioned references for details.

To learn a graph representation, we can aggregate the vertex representations using a so-called \emph{READOUT} function defined to be permutation-invariant over the labels. A graph representation $\vh_{G}$ by a GNN is then
\begin{equation*}
\vh_{G} = \mathrm{READOUT}\Big( \big\{\vh^{(t)}_v\big\}_{v, t \in V \times \{1\ldots,T\} }\Big),
\end{equation*}
where the vertex features $\vh^{(t)}_v$ are as in~\Cref{eq:app_gnn}.  $\mathrm{READOUT}$ may or may not contain learnable weights. We denote it as \emph{XU-READOUT} to not confuse with our notation $\text{READOUT}_\Gamma$.

The entire function is differentiable and can be learned end-to-end.  These models are thus typically trained with variants of Stochastic Gradient Descent.  In our work, we apply this scheme over connected induced subgraphs in the graph, making them a differentiable module in our end-to-end representation scheme.

\section{Further Related Work}
\label{sec:RelatedWorkAppendix}
This section provides a more in-depth discussion placing our work in the context of existing literature.  We explain why existing state-of-the-art graph learning methods will struggle to extrapolate, subgraph methods, and more in Graph Neural Networks literature.

\paragraph{Extrapolation.}
Geometrically, extrapolation can be thought as reasoning beyond a convex hull of a set of training points~\citep{hastie2012elements, haffner2002escaping, king2006dangers, xu2020neural}.
However, for neural networks---and their arbitrary representation mappings---this geometric interpretation is insufficient to describe a truly broad range of tasks. 
Rather, extrapolations are better described through counterfactual reasoning~\citep{neyman1923applications, rubin1974estimating, pearl2009causality, scholkopf2019causality}.

\update{As shown in~\citet{geirhos2020shortcut}, the ability of deep neural networks to capture shortcuts for predictions tends to results in poor extrapolation performances. Therefore, specific methods or strategies must be adopted to obtain extrapolation abilities.}

There are other approaches for conferring models with extrapolation abilities.  These ideas have started to permeate graph literature, which we touch on here, but remain outside the scope of our systematic counterfactual modeling framework.

Incorporating domain knowledge is an intuitive approach to learn a function that predicts adequately outside of the training distribution, data collection environment, and heuristic curation.  This has been used, for example, in time series forecasting~\citep{scott1993causal, armstrong2005decomposition}.  This can come in the form of re-expressing phenomena in a way that can be adequately and accurately represented by machine learning methods~\citep{Lample2020Deep} or specifically augmenting existing general-purpose methods to task~\citep{klicpera2020directional}.  In the context of graphs, it has been used to pre-process the graph input to make a learned graph neural network model a less complex function and thus extend beyond training data~\citep{xu2020neural}, although this does not necessarily fall into the framework we consider here.

Another way of moving beyond the training data is \emph{robustness}.  Relevant for deep learning systems are adversarial attacks~\citep{papernot2017practical}.  Neural networks can be highly successful classifiers on the training data but become wildly inaccurate with small perturbations of those training examples~\citep{goodfellow2014adv}.  This is important, say, in self-driving cars~\citep{SitawarinCars}, which can become confused by graffiti.  This becomes particularly problematic when we deploy systems to real-world environments outside the training data.  Learning to defend against adversarial attacks is in a way related to performing well outside the environment and curation heuristics encountered in training.   An interesting possibility for future work is to explore the relationships between the two approaches.

Overfitting will compromise even in-distribution generalization.  Regularization schemes such as explicit penalties are a well known and broadly applicable strategy~\citep{hastie2012elements}.  Another implicit approach is data augmentation~\citep{hernandez2018data}, and the recent GraphCrop method proposes a scheme for graphs that randomly extracts subgraphs from certain graphs in a minibatch during training~\citep{wang2020graphcrop}.  These directions differ from our own in that we seek a formulation for extrapolation even when overfitting is not necessarily a problem. Still these two approaches are both useful in the toolbox of representation learning.

We would like to point out that representation learning on \emph{dynamic graphs}~\citep{kazemi2020representation}, including tasks like link prediction on growing graphs~\citep{anonymous2021incremental}, is a mostly separate research direction from what we consider here (although it is now understood that temporal and static graph representations are equivalent for observational predictions~\citep{gao2021equivalence}).  
In these scenarios, there is a direct expectation that the process we model will change and evolve.  For instance, knowledge bases -- a form of graph encoding facts and relationships  -- are inevitably incomplete~\citep{sun-etal-2018-open}.  Simply put, developments in information and society move faster than they can be curated.  Another important example is recommendation systems~\citep{kumarLearningDynamic} based on evolving user-item networks.  These concepts are related to the counterfactuals on graphs~\citep{eckles2016design} that we discuss.  This is fundamentally different from our work where we do graph-wide learning and representation of a dataset of many graphs rather than one constantly evolving graph.

\paragraph{Subgraph methods and Graphlet Counting Kernels.}
A foundational principle of our work is that, by exploiting subgraphs, we confer graph classifications models with both the ability to fit the training data and to extrapolate to graphs from a different distribution (OOD generalization).  
As detailed in Section~\ref{sec:conditions}, this insight follows from the Aldous-Hoover representation of jointly exchangeable distributions (graphs)~\citep{hoover1979relations, aldous1981representations, kallenberg2006probabilistic, orbanz2014bayesian} and work on graph limits~\citep{lovasz2012large}.  We now discuss the larger literature that uses subgraphs in machine learning.

Counting kernels~\citep{shervashidze2009efficient} measures the similarity between two graphs by the dot product of their normalized counts of connected induced subgraphs (graphlet).  This can be used for classification via kernelized methods like Support Vector Machines (SVM). 
\mbox{\citet{yanardag2015deep}} argues that the dot product does not capture dependence between subgraphs and extend to a general bilinear form over a learned similarity matrix.  These approaches are related to the Reconstruction Conjecture, which posits graphs can be determined through knowledge of their subgraphs~\citep{kelly1957congruence, ulam1960collection, hemminger1969reconstructing, mckay1997small}.  It is known that computing a maximally expressive graph kernel, or one that is injective over the class of graphs, is as hard as the Graph Isomorphism problem, and thus intractable in general~\citep{gartner2003graph, kriege2020survey}.
\citet{kriege2018property} shows graph properties that subgraph counting kernels fail to predict. The work then proposes a method to make them more expressive, but only for graphs without vertex attributes.  

Most applications of graphlet counting  do not exploit vertex attributes, and even those that do (e.g.~\citet{wale2008comparison})  are likely to fail under a distribution shift over attributes; this is because counting each type of attributed subgraph (e.g. red clique, blue clique) is sensitive to distribution shift.  In comparison, our use of GNNs confers our framework with the ability learn a compressed representation of different attributed subgraphs, tailored for the task, and extrapolate even under attribute shift.  We demonstrate this in \Cref{tab:attr}.  
Last, a recent work~\citep{ye2020deepmap} proposes to pass the attributed subgraph counts to a downstream neural network model to better compress and represent the high dimensional feature space.  However, with attribute shifts, it may be that the downstream layers did not see enough attributed subgraph of certain types in training to learn how to correctly represent them.  We feel that it is better to \emph{compress} the attributed signal \emph{in the process of} representing the graph to handle these vertex features, the approach we take in this work.

There are many graph kernel methods that do not leverage subgraph counts but other features to measure graph similarity, such as the count of matching walks, e.g.~\citet{kashima2003marginalized, borgwardt2005protein, borgwardt2005shortest}. The WL Kernel uses the WL algorithm to compare graphs~\citep{shervashidze2011weisfeiler} and will inherit the limitations of WL GNNs like inability to represent cycles.~\citet{rieck19a} propose a  persistent WL kernel that uses ideas from Topological Data Analysis~\citep{munch2017user} to better capture such structures when comparing graphs.  Methods that do not count subgraphs will not inherit properties regarding a graph-size environment change -- from our analysis of asymptotic graph theory -- but all extrapolation tasks require an assumption and our framework can be applied to studying the ability of various kernel methods to extrapolate under different scenarios.  Those relying on attributes to build similarities are also likely to suffer from attribute shift.

Subgraphs are studied to understand underlying mechanisms of graphs like gene regulatory networks, food webs, and the vulnerability of networks to attack, and sometimes used prognostically.  A popular example investigates \emph{motifs}, subgraphs that appear more frequently than under chance~\citep{stone1992competitive, shen2002network, milo2002network, mangan2003structure, sporns2004motifs, bascompte2005simple, alon2007network, chen2013identification, benson2016higher, stone2019network, dey2019network, wang2020identifying}.    
Although the study of motifs is along a different direction and often focus on one-graph datasets, our framework learns rich latent representations of subgraphs.  
Another line of work uses subgraph counts as graph similarity measures, an example being matching real-world graphs to their most similar random graph generation models~\citep{prvzulj2007biological}.

Other machine learning methods based on subgraphs have also been proposed.  Methods like mGCMN~\citep{li2020representation}, HONE~\citep{rossi2018higher}, and MCN~\citep{lee2018higher} learn representations for vertices by extending classical methods over edges to a new neighborhood structure based on subgraphs; for instance, mGCMN runs a GNN on the new graph.  These methods do not exploit all subgraphs of size $k$ and will not learn subgraph representations in a manner consistent with our extrapolation framework.
\citet{teruhamilton2020} uses subgraphs around vertices to predict missing facts in a knowledge base.
Further examples include the Subgraph Prediction Neural network~\citep{meng2018subgraph} that predicts subgraph classes in one dynamic heterogeneous graph; counting the appearance of edges in each type of subgraph for link prediction tasks~\citep{abuoda2019link}; and SEAL~\citep{zhang2018link} runs a GNN over subgraphs extracted around candidate edges to predict whether an edge exists. While these methods exploit small subgraphs for their effective balance between rich graph information and computational tractability, they are along an orthogonal research direction.

\paragraph{Graph Neural Networks.}
Among the many approaches for graph representation learning and classification, which include methods for vertex embeddings that are subsequently read-out into graph representations~\citep{belkin2002laplacian, perozzi2014deepwalk,niepert2016learning, ou2016asymmetric, kipf2016variational, grover2016node2vec, yu2018learning, qiu2018network, maron2018invariant, maron2019provably, wu2020comprehensive, hamilton2020book, chami2020machine}, we focus our discussion and modeling on Graph Neural Network (GNN) methods~\citep{Kipf2016,Atwood2016, Hamilton2017, gilmer17a,velickovic2018graph, xu2018powerful, morris2019weisfeiler, pmlr-v97-you19b, liu2019hyperbolic, chami2019hyperbolic}. 
GNNs are trained end-to-end, can straightforwardly provide latent graph representations for graphs of any size, easily handle vertex/edge attributes, are computationally efficient, and constitute a state-of-the-art method. However, GNNs lack extrapolation capabilities due also to their inability to learn latent representations that capture the topological structure of the graph~\citep{xu2018powerful, morris2019weisfeiler, garg2020gnnRepresentation, sato2020survey}.  Relevantly, many cannot count the number of subgraphs such as triangles (3-cliques) in a graph~\citep{arvind2020weisfeiler,chen2020can}.  In general, our theory of extrapolating in graph tasks requires properly capturing graph structure. \update{In our work we consider GIN~\citep{xu2018powerful}, GCN~\citep{Kipf2016} and PNA~\citep{corso2020principal} as baseline GNN models. GIN and GCN are some of the most widely used models in literature. PNA generalizes different GNN models by considering multiple neighborhood aggregation schemes. Note that since we compare against PNA we do not need to consider other neighborhood aggregation schemes in GNNs, as those studied in~\citet{velivckovic2019neural}. To test whether more expressive models are able to extrapolate, we employ RPGIN~\citep{pmlr-v97-murphy19a}. In our experiments, we show that these state-of-the-art methods are expressive in-distribution but fail to extrapolate.}

\section{Experiments}
\label{sec:ExperimentAppendix}

In this appendix we present the details of the experimental section, discussing the hyperparameters
that have been tuned.
Training was performed on NVIDIA GeForce RTX 2080 Ti, GeForce GTX 1080 Ti, TITAN V, and TITAN Xp GPUs.

\subsection{Model implementation}

All neural network approaches, including the models proposed in this paper, are implemented in PyTorch~\citep{pytorchcitation} and Pytorch Geometric~\citep{FeyLenssen2019}.

Our GIN~\citep{xu2018powerful}, GCN~\citep{Kipf2016} and PNA~\citep{corso2020principal} implementations are based on their Pytorch Geometric implementations. We consider sum, mean, and max READOUTs as proposed by~\citet{xu2020neural} for extrapolations (denoted by {\em XU-READOUT}). For RPGIN~\citep{pmlr-v97-murphy19a}, we implement the permutation and concatenation with one-hot identifiers (of dimension 10) and use GIN as before. Other than a few hyperparameters and architectural choices,  
we use standard choices (e.g.~\citet{hu2020open}) for neural network architectures. If the graphs are unattributed, we follow convention and assign a constant ${\bf 1}$ dummy feature to every vertex. 

We use the WL graph kernel implementations provided by the \emph{graphkernels} package~\citep{Sugiyama-2017-Bioinformatics}.  All kernel methods use a Support Vector Machine on scikit-learn~\citep{scikit-learn}.  

The Graphlet Counting kernel (GC kernel), as well as our own procedure, relies on being able to efficiently count attributed or unattributed connected induced homomorphisms within the graph.  We use ESCAPE~\citep{pinar2017escape} and R-GPM~\citep{teixeira2018graph} as described in the main text.  The source code of ESCAPE is available online and the authors of~\citet{teixeira2018graph} provided us their code. We pre-process each graph beforehand and save the obtained estimated induced homomorphism densities. Note that R-GPM takes around $20$ minutes per graph in the worst case considered, but graphs can be pre-processed in parallel. ESCAPE takes up to one minute per graph.

\newcommand{\G}{\cG^\wcard_{N^\wcard}}

All the models learn graph representations $\Gamma(\G)$, which we pass to a $L$-hidden layer feedforward neural network (MLP) with softmax outputs ($L \in \{0,1\}$ depending on the task) to obtain the prediction.  For \karyGnn, and \karyRpGnn, we use respectively GIN and RPGIN as our base models to obtain latent representations for each $k$-sized connected induced subgraph. Then, we sum over the latent representations, each weighted by its corresponding induced homomorphism density, to obtain the graph representation. For \graphletCounting, the representation $\Gamma_{\text{1-hot}}(\G)$ is a vector containing densities of each (possibly attributed) $k$-sized connected subgraph. To map this into a graph representation, we apply $\Gamma_{\text{1-hot}}(\G)\transpose \mW$ where $\mW$ is a learnable weight matrix whose rows are subgraph representations.  Note that this effectively learns a unique weight vector for each subgraph type.
 
We use the Adam optimizer to optimize all the neural network models.  When an in-distribution validation set is available (see below), we use the weights that achieve best validation-set performance for prediction. Otherwise, we train for a fixed number of epochs.

The specifics of hyperparameter grids and downstream architectures are discussed in each section below.

\subsection{Schizophrenia Task: Size extrapolation}
The results of these experiments are reported in \Cref{tab:unattributed} (left).
The data was graciously provided by the authors of~\citet{de2016mapping}, which they pre-processed from publicly available data from The Center for Biomedical Research Excellence. 
There are 145 graphs which represent the functional connectivity brain networks of 71 schizophrenic patients and 74 healthy controls.  Each graph has 264 vertices representing spherical regions of interest (ROIs).  Edges represent functional connectivity.  Originally, edges reflected a time-series coherence between regions.  If the coherence between signals from two regions was above a certain threshold, the authors created a weighted edge.  Otherwise, there is no edge.  For simplicity, we converted these to unweighted edges. Extensive pre-processing must be done over fMRI data to create brain graphs.  This includes discarding signals from certain ROIs.  As described by the authors, these choices make highly significant impacts on the resulting graph.  We refer the reader to the paper~\citep{de2016mapping}.  
Note that there are numerous methods for constructing a brain graph, and in ways that change the number of vertices.  The measurement strategy taken by the lab can result in measuring about 500 ROIs, 1000 ROIs, or 264 as in the case we consider~\citep{hagmann2007mapping, wedeen2005mapping, de2016mapping}.

For our purposes, we wish to create an extrapolation task, where a change in environment leads to an extrapolation set that contains smaller graphs. 
For this, we randomly select 20 of the 145 graphs in the dataset, balanced among the healthy and schizophrenic patients, to be used as test.
For each healthy-group graph in these 20 graphs, we sample (with replacement) $\lfloor 0.4\times 264 \rfloor$ vertices to be removed. In average, the new size for the healthy-group graphs in these 20 graphs is $178.2$.

We hold out the test graphs that are later used to assess the extrapolation capabilities. Over the remaining data, we use a stratified 5-fold cross-validation to choose the hyperparameters and to report the validation accuracy.

Once the best hyperparameters are chosen, we re-train the model on the entire training data using 10 different initialization seeds, and predict on the test. 

For \karyGnn $ $ and \karyRpGnn, in their GNNs, the aggregation MLP of \Cref{eq:app_gnn} has hidden neurons chosen among $\{32, 64, 128, 256\}$ and number of layers (i.e. recursions of message-passing) among $\{1, 2\}$.
The learning rate is chosen in $\{0.001, 0.0001\}$. The value of $k$ is treated as a hyperparameter chosen in $\{4, 5\}$.

For \graphletCounting, recall that we wish to learn the matrix $\mW$ whose rows are subgraph representations. We choose the dimension of the representations among $\{32, 64, 128, 256\}$ and the learning rate in $\{0.001, 0.0001\}$. The value of $k$ is treated as a hyperparameter chosen in $\{4, 5\}$.

For the GNNs, we tune the learning rate in $\{0.01, 0.001\}$,
the number of hidden neurons of the MLP in \Cref{eq:app_gnn} in $\{32, 64, 128\}$,
the number of layers among $\{1, 2, 3\}$.

For all these models, we use a batch size of 32 graphs and a single final linear layer with a softmax activation as the downstream classifier. We optimize for 400 epochs. 

For the graph kernels, following~\citet{kriege2020survey}, we tune the regularization hyperparameter C in SVM  over the set $\{10^{-3}, 10^{-2}, 10^{-1},1, 10, 10^{2},10^{3}\}$.  We tune the number of Weisfeiler-Lehman iterations of the WL kernel to be in $\{1, 2, 3, 4\}$ (see~\citet[Section 3.1]{kriege2020survey}).

\subsection{\erdosrenyi Connection Probability: Size Extrapolation}

We simulated \erdosrenyi graphs (Gnp model) using NetworkX~\citep{SciPyProceedings_11_networkx}.
\update{The task is to classify the edge probability $p \in \{0.2, 0.5, 0.8 \}$ of the generated graph.}
\Cref{tab:unattributed} shows results for a single environment task (middle), where graphs in training have all size $80$, and a multiple environment task (right), where training graphs have sizes in $\{70, 80\}$ chosen uniformly at random. In both cases, the test is composed of graphs of size 140. The training, validation, and test sets are fixed. The number of graphs in training, validation, and test are 80, 40, and 100, respectively.
The induced homomorphism densities are obtained for subgraphs of a fixed size $k = 5$.

For \graphletCounting, we hyperparameter tune the dimension of the subgraph representations in $\{32, 64, 128, 256\}$ and the learning rate in $\{0.1, 0.01, 0.001\}$.

For the GNNs and for \karyGnn, and \karyRpGnn, we hyperparameter tune
the number of hidden neurons in the MLP of the GNN (\Cref{eq:app_gnn}) in $\{32, 64, 128, 256\}$ (GNN is used to learn the representation for $k$-sized subgraph for \karyGnn, and \karyRpGnn). The number of layers is also a hyperparameter in $\{1, 2, 3\}$ (3 layers only for the GNNs), and the learning rate in $\{0.1, 0.01, 0.001\}$. We also hyperparameter tune the presence or absence of the Jumping Knowledge mechanism from~\citet{xuJumpingKnowledge}.

For IRM, we consider the two distinct graph sizes to be the two training environments. We tune the regularizer $\lambda$~\citep[Section 3]{arjovsky2019invariant} in $\{4, 8, 16, 32\}$, stopping at 32 because increasing its value decreased performances.

We train all neural models for 500 epochs with batch size equal to the full training data. The downstream classifier is composed by a single linear layer with softmax activations. 
We perform early stopping as per~\citet{hu2020open}. 
The hyperparameter search is performed by training all models with 10 different initialization seeds and selecting the configuration that achieved the highest mean accuracy on the validation data. 
Then, we report the mean (and standard deviation) accuracy over the training, the validation, and the test data in \Cref{tab:unattributed} (right).

For the graph kernels, following~\citet{kriege2020survey}, we tune the regularization hyperparameter C in SVM  over the set $\{10^{-3}, 10^{-2}, 10^{-1},1, 10, 10^{2},10^{3}\}$.  We tune the number of Weisfeiler-Lehman iterations of the WL kernel to be among $\{1, 2, 3, 4\}$ (see~\citet[Section 3.1]{kriege2020survey}).

\subsection{Extrapolation performance over SBM attributed graphs}
We sample Stochastic Block Model graphs (SBM) using NetworkX~\citep{SciPyProceedings_11_networkx}.
Each graph has two blocks, having a within-block edge probability of $\mP_{1,1} = \mP_{2,2} = 0.2$. The cross-block edge probability is $\mP_{1,2} = \mP_{2,1} \in \{0.1, 0.3\}$.
The label of a graph is its cross-block edge probability, i.e., $Y  = \mP_{1,2}$.

Vertex color distributions change with train and test environments. In training, vertices in the first block
are either red or blue, with probabilities $\{0.9, 0.1\}$, respectively, while
vertices in the second block are either green or yellow, with probabilities $\{0.9, 0.1\}$, respectively. 
In test, the probability distributions are reversed: Vertices in the first block
are either red or blue, with probabilities $\{0.1, 0.9\},$ respectively, and vertices in
the second block are green or yellow with probabilities $\{0.1, 0.9\},$ respectively.

\Cref{tab:attr} shows results for the three scenarios we considered:
\begin{enumerate*}
    \item A single environment, where training graphs are of size 20 (left),
    \item A multiple environment, where training graphs have size 14 or 20, chosen uniformly at random (middle),
    \item A multiple environment, where training graphs are of size 20 or 30, chosen uniformly at random (right).
\end{enumerate*}
The test is the same in all cases, and contains graphs of size 40. 
The number of graphs in training, validation, and test are 80, 20, and 100, respectively.
We obtain the induced homomorphism densities for \karyGnn, \karyRpGnn, \graphletCounting $ $ for a fixed subgraph size $k = 5$.

For the GNNs and for \karyGnn $ $ and \karyRpGnn, we choose the number of hidden neurons in the MLP of the GNN (\Cref{eq:app_gnn}) in $\{32, 64, 128, 256\}$, the number of layers in $\{1, 2, 3\}$ (3 layers only for the GNNs) and hyperparameter tune the presence or absence of the Jumping Knowledge mechanism from~\citet{xuJumpingKnowledge}. We add the regularization penalty in \Cref{eq:regul} for \karyGnn $ $ and \karyRpGnn $ $ in this experiments.
For \karyGnn $ $ and \karyRpGnn $ $, we choose the learning rate in $\{0.01, 0.001\}$
and the regularization weight in $\{0.1, 0.15\}$. For the GNNs we choose the learning rate in $\{0.1, 0.01, 0.001\}$.

For IRM, we consider the two distinct graph sizes to be the two training environments. We can not treat vertex attributes as environment here since we only have a single vertex-attribute distribution in training. We tune the regularizer $\lambda$~\citep[Section 3]{arjovsky2019invariant} in $\{4, 8, 16, 32\}$, stopping at 32 because increasing its value decreased performances.

For \graphletCounting, we hyperparameter tune the dimension of the subgraph representations in $\{32, 64, 128, 256\}$ and the learning rate in $\{0.01, 0.001\}$.

We optimize all neural models for 500 epochs with batch size equal to the full training data. We use a single layer with softmax outputs as the downstream classifier.
We perform early stopping as per~\citet{hu2020open}. 
The hyperparameter search is performed by training all models with 10 different initialization seeds and selecting the configuration that achieved the highest mean accuracy on the validation data. 
Then, we report the mean (and standard deviation) accuracy over the training, the validation, and the test data in \Cref{tab:attr}.

For the graph kernels, following~\citet{kriege2020survey}, we tune the regularization hyperparameter C in SVM  over the set $\{10^{-3}, 10^{-2}, 10^{-1},1, 10, 10^{2},10^{3}\}$.  We tune the number of Weisfeiler-Lehman iterations of the WL kernel to be among $\{1, 2, 3, 4\}$ (see~\citet[Section 3.1]{kriege2020survey}).

%
%
\begin{table*}
	\caption{Dataset statistics, Table from~\citet{yehudai2020size}.}
	\label{stat}
    \begin{small}
	\begin{sc}
	\begin{center}
	\resizebox{\textwidth}{!}{
	\begin{subtable}{\textwidth}
	\centering
        \begin{tabular}{|l|r|r|r|r|r|r|}
            \cline{2-7}
            \multicolumn{1}{c|}{} & \multicolumn{3}{c|}{\textbf{NCI1}} & \multicolumn{3}{c|}{\textbf{NCI109}} \\
            \cline{2-7}
            \multicolumn{1}{c|}{} & \textbf{all} & \textbf{Smallest} $\mathbf{50\%}$ & \textbf{Largest $\mathbf{10\%}$} & \textbf{all} & \textbf{Smallest} $\mathbf{50\%}$ & \textbf{Largest $\mathbf{10\%}$} \\
            \hline
            \textbf{Class A} & $49.95\%$ & $62.30\%$ & $19.17\%$ & $49.62\%$ & $62.04\%$ & $21.37\%$ \\
            \hline
            \textbf{Class B} & $50.04\%$ & $37.69\%$ & $80.82\%$ & $50.37\%$ & $37.95\%$ & $78.62\%$ \\
            \hline
            \textbf{Num of graphs} & 4110 & 2157 & 412 & 4127 & 2079 & 421 \\
            \hline
            \textbf{Avg graph size} & 29 & 20 & 61 & 29 & 20 & 61 \\
            \hline
        \end{tabular}
    \end{subtable}
    }

    \bigskip

	\resizebox{\textwidth}{!}{
    \begin{subtable}{\textwidth}
    \centering
        \begin{tabular}{|l|r|r|r|r|r|r|}
            \cline{2-7}
            \multicolumn{1}{c|}{} & \multicolumn{3}{c|}{\textbf{PROTEINS}} & \multicolumn{3}{c|}{\textbf{DD}} \\
            \cline{2-7}
            \multicolumn{1}{c|}{} & \textbf{all} & \textbf{Smallest} $\mathbf{50\%}$ & \textbf{Largest $\mathbf{10\%}$} & \textbf{all} & \textbf{Smallest} $\mathbf{50\%}$ & \textbf{Largest $\mathbf{10\%}$} \\
            \hline
            \textbf{Class A} & $59.56\%$ & $41.97\%$ & $90.17\%$ & $58.65\%$ & $35.47\%$ & $79.66\%$ \\
            \hline
            \textbf{Class B} & $40.43\%$ & $58.02\%$ & $9.82\%$ & $41.34\%$ & $64.52\%$ & $20.33\%$ \\
            \hline
            \textbf{Num of graphs} & 1113 & 567 & 112 & 1178 & 592 & 118 \\
            \hline
            \textbf{Avg graph size} & 39 & 15 & 138 & 284 & 144 & 746 \\
            \hline
        \end{tabular}
    \end{subtable}
    }
    \end{center}
    \end{sc}
    \end{small}
\end{table*}
%
\subsection{Extrapolation performance in real world tasks that violate our causal model}
The results on graphs that violate our causal model are reported in \Cref{REAL-WORLD-MATTH}.
We use the datasets from~\citet{Morris2020}, split into train, validation and test as proposed by~\citet{yehudai2020size}.
In particular, train is obtained by considering the graphs with sizes smaller than the 50-th percentile, and test those with sizes larger than the 90-th percentile. Additionally, $10\%$ of the training graphs is held out from training and used as validation. For statistics on the datasets and corresponding splits, see~\citet{yehudai2020size}. 

We obtain the homomorphism densities for a fixed subgraph size $k = 4$. We observed that larger subgraph sizes, $k \geq 5$, implies a larger number of distinct subgraphs and consequently a smaller proportion of shared subgraphs in different graphs. 
To further reduce the number of distinct subgraphs seen by the models, we only consider the most common subgraphs in training and validation when necessary.
Specifically, for \textsc{NCI1} and \textsc{NCI109}, we only use the top 100 subgraphs (out of a total of around 300), and for \textsc{DD} only the 30k most common (out of a total of around 200k). For \textsc{PROTEINS} we keep all the distinct subgraphs (which are around 180).

For the GNNs, we follow the setup proposed in~\citet{yehudai2020size}, where all the GNNs have 3 layers and a final classifier composed of a feedforward neural network (MLP) with 1 hidden layer and softmax outputs. \update{We also use a dropout of 0.3}. We tune the batch size in $\{64, 128\}$, the learning rate in $\{0.01, 0.005, 0.001\}$ and the network width in $\{32, 64\}$. For \karyGnn $ $ and \karyRpGnn, the setup is the same, except for the number of GNN layers that is set to 2. For \textsc{DD} we use a fixed batch size of 256 to reduce the number of times the subgraphs are passed to the network, in order to speed up training.

For \graphletCounting, we choose the batch size in $\{64, 128\}$, the learning rate in $\{0.01, 0.005, 0.001\}$ and the dimension of the subgraph representations in $\{32, 64\}$.

For IRM we tune the regularizer $\lambda$~\citep[Section 3]{arjovsky2019invariant} in $\{8, 32, 128, 512\}$. The two environments are considered to be graphs with size smaller than the median size in the training graphs and larger than the median size in the training graphs, respectively.

To mitigate the imbalance between classes in training, we reweight the classes in the loss with the training proportions for each class.
We train all neural models for 1000 epochs \update{using early stopping as per~\citet{hu2020open}}. We test the models on the epoch achieving the highest mean Matthew Correlation Coefficient on validation because of the significant class imbalance in the test, see \Cref{stat}. 

For the graph kernels, following~\citet{kriege2020survey}, we tune the regularization hyperparameter C in SVM  over the set $\{10^{-3}, 10^{-2}, 10^{-1},1, 10, 10^{2},10^{3}\}$. We 
fix the number of Weisfeiler-Lehman iterations of the WL kernel to $3$ (see~\citet[Section 3.1]{kriege2020survey}), which is comparable to the $3$ GNN layers.


%



\cleardoublepage

\end{document}
